\documentclass[letterpaper]{article}
\usepackage{uai2018}
\usepackage[margin=1in]{geometry}

\usepackage{natbib}

\usepackage{tikz}
\usetikzlibrary{fit}					
\usetikzlibrary{backgrounds}	

\usepackage{url}
\usepackage{amsmath}
\usepackage{amssymb}
\usepackage{amsbsy}
\usepackage{siunitx}
\usepackage{mathtools}
\usepackage{array}
\usepackage{setspace}
\usepackage{caption}
\usepackage{subcaption}
\usepackage{booktabs}
\usepackage{adjustbox}
\usepackage{microtype}

\usepackage{wrapfig}
\usepackage{tabularx}

\usepackage{times}
\usepackage[bottom]{footmisc}
\usepackage{listings}
\usepackage{color}
\usepackage{xcolor}
\usepackage{textcomp}
\usepackage{xspace}

\usepackage{floatrow}

\usepackage{amsthm}

\usepackage{bbm}

\usepackage{todonotes}

\theoremstyle{plain}

\theoremstyle{remark}

\theoremstyle{lemma}

\usepackage{paralist}

\usepackage{thmtools}
\usepackage{thm-restate}

\usepackage{multicol}

\newcommand{\E}{\mathbb{E}}

\definecolor{darkgreenClj}{rgb}{0.25,.5,0.25}
\definecolor{blueClj}{rgb}{0,0.33,0.66}
\definecolor{redClj}{rgb}{0.66,0.0,0.0}
\definecolor{purpleClj}{rgb}{0.33,0,0.66}
\definecolor{cyanClj}{rgb}{0.0,0.5,0.5}
\definecolor{orangeClj}{rgb}{0.75,0.35,0.0}
\definecolor{grayClj}{rgb}{0.4,0.4,0.4}
\lstnewenvironment{code}[2]{\lstset{caption=#1,label=#2}}{}

\newtheorem{corollary}{Corollary}

\usepackage{abbreviations}

\usepackage{algorithmicx}
\usepackage{algorithm}
\usepackage[noend]{algpseudocode}
\usepackage{setspace}
\renewcommand{\algorithmicforall}{\textbf{for each}}
\algnewcommand{\IIf}[1]{\State\algorithmicif\ #1\ \algorithmicthen}
\algnewcommand{\ElseIIf}{\unskip\ \algorithmicelse}
\algnewcommand{\EndIIf}{\unskip\ \algorithmicend\ \algorithmicif}

\usepackage{titlesec}
\titlespacing\section{0pt}{4pt plus 2pt minus 2pt}{0pt plus 2pt minus 0pt}
\titlespacing\subsection{0pt}{4pt plus 2pt minus 2pt}{0pt plus 2pt minus 0pt}
\titlespacing\subsubsection{0pt}{4pt plus 2pt minus 2pt}{0pt plus 2pt minus 0pt}

\usepackage{times}
\title{Nesting Probabilistic Programs}

\author{
	Tom Rainforth \\
	Department of Statistics\\
	University of Oxford\\
	\texttt{\small rainforth@stats.ox.ac.uk} \\
}

\begin{document} 
	
	
\maketitle	
	
\setlength{\abovedisplayskip}{2.5pt}
\setlength{\belowdisplayskip}{2.5pt}
\setlength{\abovedisplayshortskip}{2.5pt}
\setlength{\belowdisplayshortskip}{2.5pt}


\begin{abstract}
\vspace{-10pt}
We formalize the notion of nesting probabilistic programming queries
 and investigate the resulting statistical implications.  We demonstrate
that while query nesting allows the definition of models which could not 
otherwise be expressed, such as those involving agents reasoning about other agents,
existing systems take approaches which lead to
inconsistent estimates.
We show how to correct this by delineating possible ways one might want to nest 
queries and asserting the respective conditions required for convergence.
We further introduce a new \emph{online} nested Monte Carlo estimator that 
makes it substantially easier to ensure these conditions are met, thereby
providing a simple framework for designing statistically correct inference
engines.  We prove the correctness of this online estimator and show that, when using the
recommended setup, its asymptotic variance is always better than that of the
equivalent fixed estimator, while its bias is always within a factor of two.
\end{abstract}

\vspace{-7pt}


\section{INTRODUCTION}
\label{sec:intro}

\vspace{0pt}

Probabilistic programming systems (PPSs) allow probabilistic models to be represented in the 
form of a generative model and statements for conditioning on data~\citep{goodman2008church,gordon2014probabilistic}.
Informally, one can think of the generative model as the definition of a prior, the conditioning 
statements as the definition of a likelihood, and the output of the program as samples from a posterior distribution. 
Their core philosophy is to decouple model specification and inference, the former 
corresponding to the user-specified program code and the latter to an inference engine 
capable of operating on arbitrary programs.  Removing the need for users to write 
inference algorithms significantly reduces the burden of developing new models and 
makes effective statistical methods accessible to non-experts.  

Some, so-called universal, systems~\citep{goodman2008church,goodman_book_2014,mansinghka2014venture,wood2014new} 
further allow the definition of models that would be hard, or
even impossible, to convey using conventional frameworks such as graphical models.
One enticing manner they do this is by allowing
arbitrary nesting of models, known in the probabilistic programming literature as queries~\citep{goodman2008church},
such that it is easy to define and run problems that fall outside the standard inference
framework~\citep{goodman2008church,mantadelis2011nesting,stuhlmuller2014reasoning,
le2016nested}. 
This allows the definition of models that could not be encoded without nesting, such as experimental 
design problems~\citep{ouyang2016practical} and various models for theory-of-mind~\citep{stuhlmuller2014reasoning}.
 In particular, models that involve 
agents reasoning about other agents require, in general,
some form of nesting.
For example, one might use such nesting to model a poker player reasoning 
about another player as shown in Section~\ref{sec:mot}.
As machine learning increasingly starts to try and tackle problem
domains that require interaction with humans or other external systems, such as
the need for self-driving cars to account for the behavior of pedestrians, we believe
that such nested problems are likely to become increasingly common
and that PPSs will form a powerful tool for encoding them.

However, previous work 
has, in general, 
implicitly, and incorrectly, assumed that the convergence results from standard inference
schemes carry over directly to the nested setting.  In truth,
inference for nested queries falls outside the
scope of conventional proofs and so additional
work is required to prove the consistency of PPS inference engines for nested queries.
Such problems constitute special cases of \textbf{\emph{nested estimation}}.
In particular, the use of Monte Carlo (\mc) methods by most PPSs mean they form
 particular instances 
of \textbf{\emph{nested Monte Carlo}} (NMC) estimation~\citep{hong2009estimating}.
Recent work~\citep{rainforth2016pitfalls,rainforth2017pitfalls,fort2016mcmc}
has demonstrated that NMC is consistent for a general class of models,
but also that it entails a convergence rate in the total computational cost
which decreases exponentially
with the depth of the nesting.  Furthermore,  additional assumptions are 
required to achieve this convergence, most noticeably that, except in a few special cases,
one needs to drive not only the total number of samples used to infinity, but also
the number of samples used at each layer of the estimator,
a requirement generally flaunted by existing PPSs.

The aim of this work is to formalize the notion of query nesting
and use these recent NMC results
 to investigate the statistical correctness of the resulting procedures
 carried out by PPS inference engines.  To do this, 
we postulate that there are three distinct ways one might nest one query within another:
sampling from the conditional distribution of another query (which
we refer to as \textbf{\emph{nested inference}}),
factoring the
trace probability of one query with the partition function estimate of another
(which we refer to as \textbf{\emph{nested conditioning}}),
and using expectation estimates calculated using one query as 
\textbf{\emph{first class variables}} in another.  
We use the aforementioned NMC results
to assess the relative correctness of each of these categories of nesting.
In the interest of exposition, we will mostly focus on the PPS
\emph{Anglican}~\citep{tolpin2016design,wood2014new}
(and also occasionally Church~\citep{goodman2008church})
as a basis for our discussion,
but note that our results apply more generally.  For example, our
nested inference case covers the problem of sampling from
cut distributions in OpenBugs~\citep{plummer2015cuts}.

We find that nested inference is statistically challenging and incorrectly 
handled by existing
systems, while nested conditioning is 
statistically straightforward and done correctly.
Using estimates as variables turns out
to be exactly equivalent to generic NMC estimation and must thus be dealt with on
a case-by-case basis.
Consequently, we will focus more on nested inference than the other cases.

To assist in the development of consistent approaches, we further
introduce a new \emph{online} NMC (ONMC) scheme that
obviates the need to revisit previous samples when refining estimates, thereby
 simplifying the
process of writing consistent online nested estimation schemes,
as required by most PPSs.
 We show that ONMC's convergence rate only varies by a small
 constant factor relative to conventional NMC: given
some weak assumptions and the use of recommended parameter settings,
its asymptotic variance is always
better than the equivalent NMC estimator with matched total sample budget, 
while its asymptotic bias is
always within a factor of two.

\section{BACKGROUND}

\subsection{NESTED MONTE CARLO}
\label{sec:nmc}

We start by providing a brief introduction to NMC, using similar notation to
that of~\cite{rainforth2017pitfalls}.
Conventional \mc estimation approximates an intractable expectation $\gamma_0$ of a 
function $\lambda$ using
\begin{align}
\label{eq:MC}
\gamma_0 &= \E\left[\lambda(y^{(0)})\right]
\approx I_0 = \frac{1}{N_0} \sum_{n=1}^{N_0} \lambda(y_n^{(0)})
\end{align}
where $y_n^{(0)} \iid p(y^{(0)})$, resulting in a mean squared error (MSE) that decreases at a rate $O(1/N_0)$.
For nested estimation
problems, $\lambda(y^{(0)})$ is itself intractable, corresponding to a
nonlinear mapping of a (nested) estimation.  Thus in the single nesting case,
$\lambda(y^{(0)}) = f_0\left(y^{(0)},
\E \left[f_1\left(y^{(0)},y^{(1)}\right)\middle|y^{(0)}\right]\right)$ giving
\begin{align*}
\gamma_0 &= \E \left[f_0\left(y^{(0)},\E \left[f_1\left(y^{(0)},y^{(1)}\right)\middle|y^{(0)}\right] \right)\right] \\
&\approx I_0 = \frac{1}{N_0} \sum_{n=1}^{N_0} f_0\left(y_n^{(0)}, \frac{1}{N_1} \sum_{m=1}^{N_1} 
f_1\left(y_n^{(0)},y_{n,m}^{(1)}\right)\right)
\end{align*}
where each $y_{n,m}^{(1)} \sim p(y^{(1)}|y_n^{(0)})$ is drawn independently and $I_0$ is now
a NMC estimate using $T=N_0 N_1$ samples.

More generally, one may have multiple layers of nesting.  To notate this, we first presume some fixed 
integral depth $D \ge 0$ (with $D=0$ corresponding to conventional estimation), and real-valued functions $f_0,\dots,f_D$.
We then recursively define
\begin{align*}
&\gamma_D\left(y^{(0:D-1)}\right) = \E \left[f_D\left(y^{(0:D)}\right) \middle| y^{(0:D-1)}\right], \quad \text{and} \\
&\gamma_k(y^{(0:k-1)}) = \E \left[f_k\left(y^{(0:k)}, \gamma_{k+1}\left(y^{(0:k)}\right)\right) \middle| y^{(0:k-1)}\right] 
\end{align*}
for $0 \leq k < D$. Our goal is to estimate $\gamma_0 = \E \left[f_0\left(y^{(0)},
\gamma_1\left(y^{(0)}\right)\right)\right]$, for which the NMC estimate is $I_0$ defined
recursively using
\begin{align}
I_D&\left(y^{(0:D-1)}\right) = \frac{1}{N_D} \sum_{n_D=1}^{N_D} f_D\left(y^{(0:D-1)}, y^{(D)}_{n_D}\right) \quad \text{and}
\nonumber \displaybreak[0]\\
I_k &\left(y^{(0:k-1)}\right) \label{eq:nmc} \\
&= \frac{1}{N_k} \sum_{n_k=1}^{N_k} f_k\left(y^{(0:k-1)}, y^{(k)}_{n_k}, I_{k+1}\left(y^{(0:k-1)}, y^{(k)}_{n_k}\right)\right)
\nonumber
\end{align}
for $0 \leq k < D$, where each $y^{(k)}_n \sim p\left(y^{(k)}|y^{(0:k-1)}\right)$ is drawn
independently. Note that there are multiple values of $y^{(k)}$ for each 
associated $y^{(0:k-1)}$
and that $I_k\left(y^{(0:k-1)}\right)$ is still a random variable given  $y^{(0:k-1)}$.

As shown by~\cite[Theorem 3]{rainforth2017pitfalls}, if each $f_k$ is continuously differentiable 
and 
\begin{align*}
\varsigma_{k}^2  
&=\E \left[\left(f_k\left(y^{(0:k)},\gamma_{k+1}
\left(y^{(0:k)}\right) \right) \hspace{-2pt}-\hspace{-2pt}\gamma_k\left(y^{(0:k-1)}\right)\right)^2\right]
\end{align*}
$< \infty \,\, \forall k\in 0,\dots,D$, then the MSE converges at rate
\begin{align}
\label{eq:bound-cont}
\begin{split}
&\E \left[\left(I_0 - \gamma_0\right)^2\right] \le 
\frac{\varsigma_0^2}{N_0}+\\
&\left(
\frac{C_0 \varsigma_{1}^2}{2 N_{1}}
+\sum_{k=0}^{D-2}  \left(\prod_{d=0}^{k} K_{d}\right)
\frac{C_{k+1} \varsigma^2_{k+2}}{2 N_{k+2}}
\right)^2 + O(\epsilon)
\end{split}
\end{align}
where $K_k$ and $C_k$ are respectively bounds on the magnitude of the first and
second derivatives of $f_k$, and $O(\epsilon)$ represents asymptotically 
dominated terms -- a convention we will use throughout.
Note that the dominant terms in the bound correspond 
respectively to the variance and the bias squared.
Theorem 2 of~\cite{rainforth2017pitfalls} further shows that the continuously differentiable
assumption must hold almost surely, rather than absolutely, for convergence more generally,
such that  functions with measure-zero discontinuities still converge in general.

We see from~\eqref{eq:bound-cont} that if any of the $N_k$ remain fixed, there is a minimum
error that can be achieved: convergence requires each $N_k\rightarrow\infty$.
As we will later show,
many of the shortfalls in dealing with nested queries by existing PPSs revolve around
implicitly fixing $N_k \, \forall k\ge1$. 

For a given total sample budget $T=N_0N_1\dots N_D$, the bound is tightest when
$\sqrt{N_0} \propto N_1 \propto \dots \propto N_D$ giving a convergence rate of $O(1/T^{\frac{2}{D+2}})$.  
The intuition behind this potentially surprising optimum setting is that the variance is
mostly dictated by $N_0$ and bias by the other $N_k$.
We see that the convergence
rate diminishes exponentially with $D$.  However, this optimal setting of the $N_k$
still gives a substantially faster rate than the $O(1/T^{\frac{1}{D+1}})$ from 
na\"{i}vely setting $N_0 \propto N_1 \propto \dots \propto N_D$.

\subsection{THE ANGLICAN PPS}
\label{sec:anglican}

Anglican is a universal probabilistic programming language integrated into \emph{Clojure}~\citep{hickey2008clojure}, a dialect of Lisp.
There are two important ideas to understand for reading Clojure: 
almost everything is a function and parentheses cause evaluation.  
For example, $a+b$ is coded as
 {\small \lsi{(+ a b)}} where \clj{+} is a function taking two arguments
and the parentheses cause the
function to evaluate.

Anglican inherits most of the syntax of Clojure, but extends it with the key
special forms \sample and \observe \citep{wood2014new,tolpin2015probabilistic,tolpin2016design}, between which the distribution of the
query is defined. 
Informally, \sample specifies terms in the prior and \observe terms in the
likelihood.  More precisely, \sample is used to make random draws from a provided
distribution and \observe is used to apply conditioning, factoring the probability density
of a program trace by a provided density evaluated at an ``observed'' point. 

The syntax of \sample is to take a \emph{distribution object} as its only input and return a sample. \observe instead
takes a distribution object and an observation and returns {\small \lsi{nil}}, while changing
the program trace probability in Anglican's back-end.  
Anglican provides a number of \emph{elementary random procedures}, i.e. distribution object constructors for common sampling distributions, 
but also allows users to define their own distribution object constructors using the \defdist macro.
Distribution objects are generated by calling a class constructor
with the required parameters, e.g. {\small \lsi{(normal 0 1)}}.

Anglican queries are written using the macro \defquery.  This allows users to define a model using a mixture
of \sample and \observe statements and deterministic code, and bind that model to a variable.  
As a simple example,
\vspace{-5pt}
	\begin{lstlisting}[basicstyle=\ttfamily\small,frame=none]
(defquery my-query [data]
 (let [$\mu$ (sample (normal 0 1))
       $\sigma$ (sample (gamma 2 2))
       lik (normal $\mu$ $\sigma$)]
  (map (fn [obs] (observe lik obs)) data)
  [$\mu$ $\sigma$]))
	\end{lstlisting}	
	\vspace{-8pt}
corresponds to a model where we are trying to infer the mean and standard deviation
of a Gaussian given some data.  The syntax of \defquery is {\small \lsi{(defquery name [args] body)}} such
that we are binding the query to \clj{my-query} here.  The query starts by sampling $\mu\sim\mathcal{N}(0,1)$
and $\sigma\sim\Gamma(2,2)$, before constructing a distribution object 
{\small \lsi{lik}}
to use for the observations.  It then maps over each datapoint and observes it under the distribution
{\small \lsi{lik}}.
After the observations are made, $\mu$ and $\sigma$
are returned from the variable-binding \cllet block and then by proxy the query itself.
Denoting the data as $y_{1:S}$ this particular query defines the joint distribution
\[
p(\mu,\sigma,y_{1:S})
= \mathcal{N}(\mu ; 0,1) \; \Gamma(\sigma ; 2, 2) \prod\nolimits_{s=1}^{S} \mathcal{N}(y_s ; \mu, \sigma).
\]
%


Inference on a query is performed using the macro \doquery, which produces a lazy infinite sequence of 
approximate samples from the conditional distribution and, for appropriate inference algorithms,
an estimate of the partition function.
Its calling syntax is {\small \lsi{(doquery inf-alg model inputs & options)}}.

Key to our purposes is Anglican's ability to \emph{nest} queries within one another.
In particular, the special form \conditional 
takes a query and returns a distribution object constructor, the outputs of which 
ostensibly corresponds to the conditional distribution 
defined by the query, with the inputs to the query becoming its parameters.  
However, as we will show in the next section, the true behavior of \conditional
deviates from this, thereby leading to inconsistent nested inference schemes.


\section{NESTED INFERENCE}
\label{sec:samp}

One of the clearest ways one might want to nest queries is by sampling from the conditional
distribution of one query inside another.  A number of examples of this are provided
for Church 
in~\citep{stuhlmuller2014reasoning}.\footnote{Though their nesting 
	happens within the conditioning predicate, Church's semantics
	means they constitute nested inference.}
Such \emph{nested inference} problems fall under a more general framework of inference for so-called
doubly (or multiply) intractable distributions~\citep{murray2006mcmc}.
The key feature of these problems is that they include
terms with unknown, \emph{parameter dependent}, normalization constants.  For nested probabilistic programming
queries, this manifests through \emph{conditional normalization}.


Consider the following unnested model using
the Anglican function declaration \defm
\vspace{-4pt}
\begin{lstlisting}[basicstyle=\ttfamily\small,frame=none]
(defm inner [y D]
 (let [z (sample (gamma y 1))]
  (observe (normal y z) D)
  z))
  
(defquery outer [D]
 (let [y (sample (beta 2 3))
       z (inner y D)]
  (* y z)))
\end{lstlisting}
\vspace{-8pt}
Here \clj{inner} is simply an Anglican function: it takes in inputs \clj{y} and \clj{D},
effects the trace probability through its \observe statement, and returns
the random variable \clj{z} as output.  The unnormalized distribution for
this model is thus straightforwardly given by
\begin{align*}
\pi_u(y,z,&D) = p(y)p(z|y)p(D|y,z) \\
=&\textsc{Beta}(y;2,3) \, \Gamma(z;y,1) \, \mathcal{N}(D;y,z^2),
\end{align*}
for which we can use conventional inference schemes.

We can convert this model to a nested inference problem by
using \defquery and \conditional as follows
\vspace{-4pt}
\begin{lstlisting}[basicstyle=\ttfamily\small,frame=none]
(defquery inner [y D]
 (let [z (sample (gamma y 1))]
  (observe (normal y z) D)
  z))
  
(defquery outer [D]
 (let [y (sample (beta 2 3))
       dist (conditional inner)
       z (sample (dist y D))]
  (* y z)))
\end{lstlisting}
\vspace{-8pt}
This is now a nested query: a separate inference procedure is invoked for each call of
\clj{(sample (dist y D))}, returning an approximate sample
from the conditional distribution defined by \clj{inner} when 
input with the current values of
\clj{y} and \clj{D}.  Mathematically, \conditional applies a conditional normalization.
Specifically, the component of $\pi_u$ from the previous example corresponding to \clj{inner} was $p(z|y)p(D|y,z)$ and \conditional locally normalizes this to the probability distribution
$p(z|D,y)$.  The distribution now defined by \clj{outer} is thus given by
\begin{align*}
\pi_n(y,z,D) &= p(y)p(z|y,D)
= \frac{p(y)p(z|y)p(D|y,z)}{\int p(z|y)p(D|y,z)dz} \displaybreak[0] \\
&= p(y)\frac{p(z|y)p(D|y,z)}{p(D|y)} \neq \pi_u(z,y,D).
\end{align*}
Critically, the partial normalization constant $p(D|y)$ depends on $y$ and so 
the conditional distribution is doubly intractable: we cannot evaluate 
$\pi_n(y,z,D)$  exactly.

Another way of looking at this
is that wrapping \clj{inner} in \conditional has ``protected'' $y$ from the conditioning 
in \clj{inner} (noting
$\pi_u (y,z,D) \propto p(y|D)p(z|y,D)$), such that its \observe statement only affects the probability
of $z$ given $y$ and not the marginal probability of $y$.  This is why, when there is only a single layer of nesting,
nested inference is equivalent to the notion of sampling from
``cut distributions''~\citep{plummer2015cuts}, whereby the sampling of certain subsets of the variables in a model are made with factors of the overall likelihood
omitted.

It is important to note that if we had observed the \emph{output} of the inner query, rather than sampling
from it, this would still constitute a nested inference problem.  The key to the nesting
is the conditional normalization applied by \conditional, not the exact usage of the generated 
distribution object \clj{dist}.  However, as discussed in Appendix~\ref{sec:app:nest-obv}, actually 
observing a nested query requires numerous additional computational issues to
be overcome, which are beyond the scope of this paper.  We thus focus on the nested sampling scenario.

\subsection{MOTIVATING EXAMPLE}
\label{sec:mot}

Before jumping into a full formalization of nested inference, we first consider
the motivating example of modeling a poker player who reasons about another player.
Here each player has access to information the other does not, namely the cards 
in their hand, and they must perform their own inference to deal with the resulting uncertainty.

Imagine that the first player is deciding whether or not to bet.  She could na\"{i}vely just make this
decision based on the strength of her hand, but more advanced play requires her to reason about
actions the other player might take given her own action, e.g. by considering
whether a bluff is likely to be successful.
She can carry out
such reasoning by constructing a model for the other player to try and predict 
their action given her action and their hand.  Again this nested model could just simply be based
on a na\"{i}ve simulation, but we can refine it by adding another layer of meta-reasoning:
the other player will themselves try to infer the first player's hand
to inform their own decision.

These layers of meta-reasoning create a nesting: for the first
player to choose an action, they must run multiple simulations 
for what the other player will do given that action and their hand, each of
which requires inference to be carried out.  Here adding more levels of meta-reasoning
can produce smarter models, but also requires additional layers of nesting.
We expand on this example to give a concrete nested inference problem in Appendix~\ref{sec:poker}.



\subsection{FORMALIZATION}

To formalize the nested inference problem more generally, let $y$ and $x$ denote
all the random variables
of an outer query that are respectively passed or not to the inner query.
Further, let $z$ denote all random variables generated in the inner query -- for simplicity,
we will assume, without loss of generality, that these are all returned to the outer
query, but that
some may not be used.  The unnormalized density for the outer query
can now be written in the form
 \begin{align}
 \label{eq:outer-q}
 \pi_o(x,y,z) &= \psi(x,y,z)  p_i(z|y)
 \end{align}
 where $p_i(z|y)$ is the normalized density of the outputs of the inner query
 and $\psi(x,y,z)$ encapsulates all other terms influencing the trace probability 
 of the outer query.  Now the inner query defines an unnormalized density
 $\pi_i (y,z)$ that can be evaluated pointwise and we have
 \begin{align}
 p_i (z|y) = \frac{\pi_i (y,z)}{\int \pi_i (y,z') dz'} \quad \text{giving}
 \end{align}
  \begin{align}
 p_o(x,y,z) \propto \pi_o(x,y,z) &= \frac{\psi(x,y,z) \pi_i (y,z)}{\int \pi_i (y,z') dz'}
 \end{align}
 where $p_o(x,y,z)$ is our target distribution, for which 
 we can directly evaluate the numerator, but the denominator is
 intractable and must be evaluated separately for each possible value of
 $y$.  Our previous example is achieved by fixing
 $\psi(x,y,z) = p(y)$ and $\pi_i (y,z) = p(z|y)p(D|y,z)$.  We can further
 straightforwardly extend to the multiple layers of nesting setting by
 recursively defining $\pi_i (y,z)$ in the same way as $\pi_o(x,y,z)$.

\subsection{RELATIONSHIP TO NESTED ESTIMATION}

To relate the nested inference problem back to 
the nested estimation formulation from Section~\ref{sec:nmc}, we consider using a proposal
$q(x,y,z)=q(x,y)q(z|y)$ to calculate the expectation of some arbitrary function 
$g(x,y,z)$ under $p_o(x,y,z)$ as per self-normalized
importance sampling
\begin{align}
&\E_{p_o(x,y,z)} \left[g(x,y,z)\right]
= \frac{\E_{q (x,y,z)} \left[\frac{g(x,y,z) \pi_o(x,y,z)}{q (x,y,z)}\right]}
{\E_{q (x,y,z)} \left[\frac{\pi_o(x,y,z)}{q (x,y,z)}\right]} \nonumber \displaybreak[0]\\
&=\cfrac{\E_{q (x,y,z)} \left[\cfrac{g(x,y,z) \psi(x,y,z) \pi_i (y,z)}
	{q (x,y,z)\E_{z' \sim q (z|y)} \left[\pi_i (y,z') / q (z'|y)\right]}\right]}
{\E_{q (x,y,z)} \left[\cfrac{\psi(x,y,z) \pi_i (y,z)}
	{q (x,y,z)\E_{z' \sim q (z|y)} \left[\pi_i (y,z') / q (z'|y)\right]}\right]}. \label{eq:nest-inf}
\end{align}
Here both the denominator and numerator are nested expectations with a
nonlinearity coming from the fact that we are using the reciprocal of
an expectation.
A similar reformulation could also be applied in cases with
multiple layers of nesting, i.e. where \clj{inner} itself makes use of
another query.  The formalization can also be directly extended to the sequential MC
(SMC) setting by invoking extended space arguments~\citep{andrieu2010particle}.

Typically $g(x,y,z)$ is not known upfront and we instead return an
empirical measure from the program in the form of weighted
samples which can later be used to estimate an expectation.  That is, if we sample $(x_n,y_n) \sim q(x,y)$ and
$z_{n,m} \sim q(z|y_n)$ and return all samples $(x_n,y_n,z_{n,m})$ (such that
each $(x_n,y_n)$ is duplicated $N_1$ times in the sample set)
then our
unnormalized weights are given by
\begin{align}
w_{n,m} = 
\frac{\psi(x_n,y_n,z_{n,m}) \pi_i (y_n,z_{n,m})}{q (x_n,y_n,z_{n,m}) 
\frac{1}{N_1}\sum_{\ell=1}^{N_1} \frac{\pi_i (y_n,z_{n,\ell})}{q (z_{n,\ell}|y_n)}}.
\label{eq:weights}
\end{align}
This, in turn, gives us the empirical measure
\begin{align}
\hat{p}(\cdot) = \frac{\sum_{n=1}^{N_0} \sum_{m=1}^{N_1}
w_{n,m} \delta_{(x_n,y_n,z_{n,m})}(\cdot)}{\sum_{n=1}^{N_0} \sum_{m=1}^{N_1} w_{n,m}}
\label{eq:emp-measure}
\end{align}
where $\delta_{(x_n,y_n,z_{n,m})}(\cdot)$ is a delta function centered on $(x_n,y_n,z_{n,m})$.
By definition, the convergence of this empirical measure to the target
requires that
expectation estimates calculated using it converge in
probability for \emph{any} integrable $g(x,y,z)$ (presuming our proposal
is valid).  We thus see that
the convergence of the ratio of nested
expectations in~\eqref{eq:nest-inf} for any arbitrary $g(x,y,z)$, is equivalent
to the produced samples converging to the distribution defined by the program.  Informally,
the NMC results then tell us this will happen in the limit $N_0, N_1 \to \infty$
provided that $\int \pi_i (y,z) dz$ is strictly positive for all possible $y$ (as otherwise
the problem becomes ill-defined).
More formally we have the following result.  Its proof, along with all others,
is given in Appendix~\ref{sec:proofs}.
\begin{restatable}{theorem}{nestinf}
	\label{the:nestinf}
Let $g(x,y,z)$ be an integrable function, let $\gamma_0 = \E_{p_o(x,y,z)} [g(x,y,z)]$,
and let $I_0$ be a self-normalized MC estimate for $\gamma_0$ calculated using $\hat{p}(\cdot)$
as per \eqref{eq:emp-measure}.
Assuming that $q(x,y,z)$ forms a valid importance sampling proposal distribution for $p_o(x,y,z)$,
then
\begin{align}
\label{eq:nest-inf-rate}
\E \left[\left(I_0-\gamma_0\right)^2\right] = \frac{\sigma^2}{N_0}+\frac{\delta^2}{N_1^2}+O(\epsilon)
\end{align}
where $\sigma$ and $\delta$ are constants derived in the proof and, as before,
$O(\epsilon)$ represents asymptotically dominated terms.
\end{restatable}	
Note that rather than simply being a bound, this result is an equality
and thus provides the exact asymptotic rate.   
Using the arguments of~\cite[Theorem 3]{rainforth2017pitfalls}, it can be straightforwardly
extended to cases of multiple nesting (giving a rate analogous to~\eqref{eq:bound-cont}), though characterizing $\sigma$ and
$\delta$ becomes more challenging.

\subsection{CONVERGENCE REQUIREMENTS}

We have demonstrated that the problem of nested inference is a particular case of
nested estimation.  
This problem equivalence will hold whether we elect to use 
the aforementioned nested importance sampling based approach or not, 
while we see that our finite
sample estimates must be biased for non-trivial $g$
by the convexity of $f_0$ and Theorem 4 of~\cite{rainforth2017pitfalls}.
Presuming we cannot produce exact samples from the inner query
and that the set of possible inputs to the inner query is not finite (these are respectively considered in Appendix~\ref{sec:exact} and Appendix~\ref{sec:special}), we thus
see that there is no ``silver bullet'' that can reduce the problem to a standard
estimation.

We now ask, what behavior do we need for Anglican's \conditional, 
and nested inference more generally, to ensure convergence? 
At a high level, the NMC results show us that we need the computational
budget of each call of a nested query to become arbitrarily large,
such that we use an infinite number of samples at each layer of the estimator:
we require each $N_k\to\infty$.

We have formally demonstrated convergence when this requirement is
satisfied and the previously introduced nested importance sampling
approach is used.  
Another possible approach would be to, instead of drawing samples to
estimate~\eqref{eq:nest-inf} directly, importance sample $N_1$
times for each call of the inner query and then return a single sample from
these, drawn in proportion to the inner query importance weights.  
We can think of this as drawing the same raw samples, but then
constructing the estimator as 
\begin{align}
\hat{p}^*(\cdot) &= \frac{\sum_{n=1}^{N_0}
	w_{n}^* \delta_{(x_n,y_n,z_{n,m^*(n)})}(\cdot)}{\sum_{n=1}^{N_0} w_{n}^*}
\label{eq:emp-measure-nonRB} \\
\text{where} \quad w^*_n &= \frac{\psi(x_n,y_n,z_{n,m^*(n)})}{q(x_n,y_n)} \quad
\text{and} \\
m^*(n) \sim \,&\textsc{Discrete}\left(\cfrac{
	\pi_i (y_n,z_{n,m})/q (z_{n,m}|y_n)}
{\sum_{\ell=1}^{N_1}\pi_i (y_n,z_{n,\ell})/q (z_{n,\ell}|y_n)}\right) \nonumber
\end{align}
As demonstrated formally in Appendix~\ref{sec:proofs}, this approach
also converges.  However, if we Rao Blackwellize~\citep{casella1996rao} 
the sampling of $m^*(n)$,
we find that this recovers~\eqref{eq:emp-measure}.  Consequently, this
is a strictly inferior estimator (it has an
increased variance relative to~\eqref{eq:emp-measure}).  
Nonetheless, it may often be a convenient setup from the perspective of
the PPS semantics and it will typically have substantially reduced
memory requirements: we need only store the single returned sample from
the inner query to construct our empirical measure, rather than all of the
samples generated within the inner query.

Though one can use the results of~\cite{fort2016mcmc} to show the correctness of instead using an MCMC estimator for the outer query, the correctness of using MCMC methods for the
inner queries is not explicitly covered by existing results.
Here we find that we need to start a new Markov chain for each call of the inner query
because each value of $y$ defines a different local inference problem.
One would intuitively expect the NMC results to carry over -- as $N_1\to\infty$ 
all the inner queries will run their Markov chains for an infinitely long time, 
thereby in principle returning exact samples
-- but we leave formal proof of this case to future work.
We note that such an approach effectively equates to what
is referred to as \emph{multiple imputation} by~\cite{plummer2015cuts}.

\subsection{SHORTFALLS OF EXISTING SYSTEMS}

Using the empirical measure~\eqref{eq:emp-measure} provides one
possible manner of producing a consistent estimate of our target
by taking $N_0,N_1 \to \infty$ and so we can use this as a gold-standard
reference approach (with a large value of $N_1$) to assess
whether Anglican returns samples for the correct target distribution.
To this end, we ran Anglican's importance sampling inference
engine on the simple model introduced earlier and compared its output
to the reference approach using $N_0 = 5\times 10^6$ and $N_1=10^3$.
As shown in Figure~\ref{fig:hists}, the samples produced by Anglican are 
substantially different to
the reference code, demonstrating that the outputs do not match 
their semantically intended distribution.
For reference, we also considered the distribution induced by the
aforementioned unnested model and a na\"{i}ve estimation scheme where a sample budget of $N_1=1$ is used for each call
to \clj{inner}, effectively corresponding to ignoring the \observe statement by directly returning the
first draw of $z$.  

We see that the unnested model defines a noticeably different distribution, while the behavior
of Anglican is similar, but distinct, to ignoring the \observe statement in the inner query.  Further
investigation shows that the default behavior of \conditional in a query nesting context is equivalent to using~\eqref{eq:emp-measure-nonRB} 
but with $N_1$ held fixed to at $N_1=2$, inducing a substantial bias. 
More generally, the Anglican source code shows that \conditional defines a Markov chain
generated by equalizing the output of the weighted samples generated by running
inference on the query.
When used to nest
queries, this Markov chain is only ever run for a finite length of time, specifically one accept-reject step is carried out, and so does
not produce samples from the true conditional distribution.


\begin{figure}[t]
	\centering
	{\includegraphics[width=0.95\textwidth,trim={1.5cm 0 3.5cm 2cm},clip]{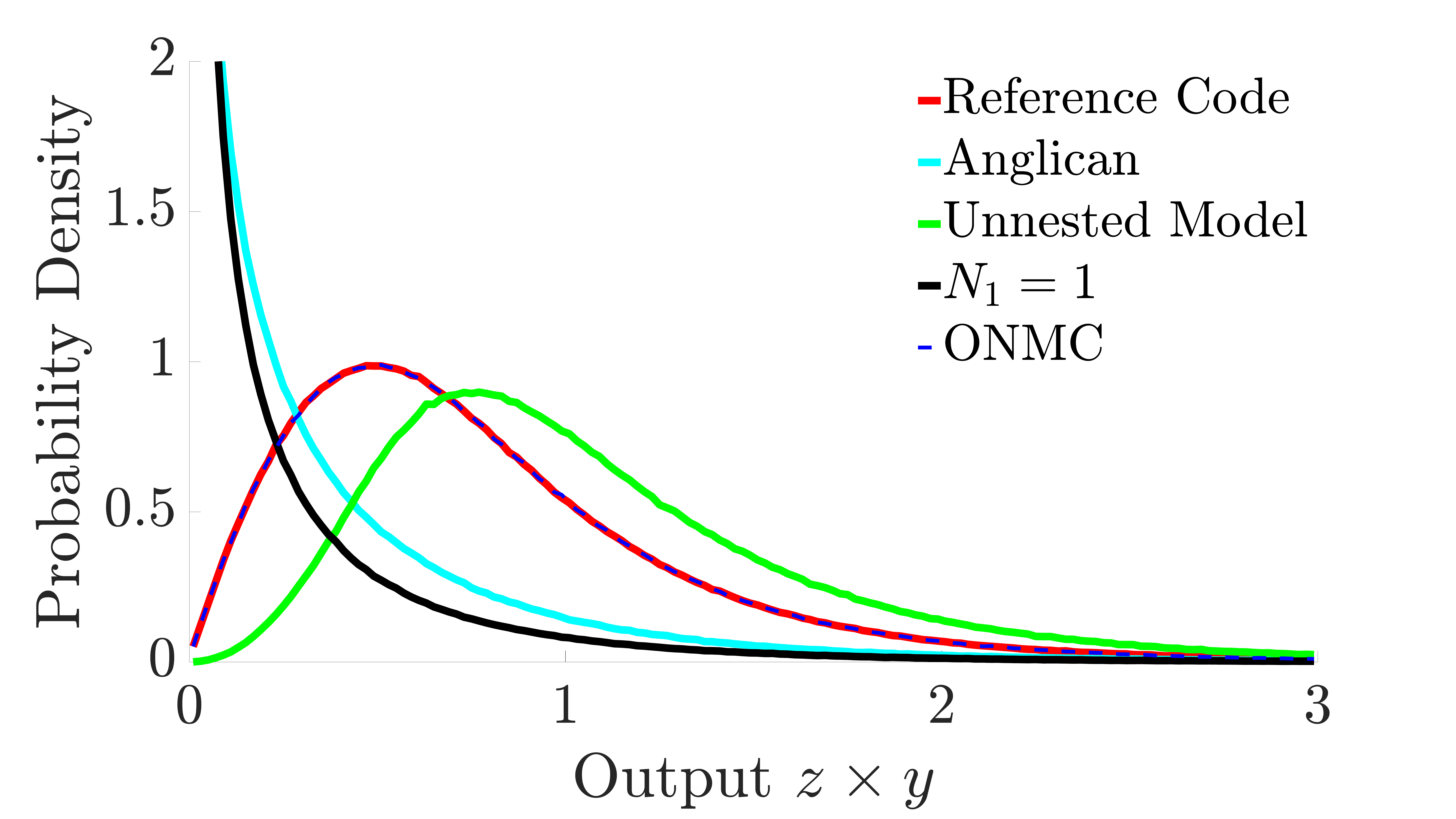}}
	\vspace{-10pt}
	\caption{Empirical densities produced by running the nested Anglican
			queries given in the text, a reference NMC estimate,
			the unnested model,
			a na\"{i}ve estimation scheme where $N_1=1$, and 
			the ONMC approach introduced in Section~\ref{sec:ONMC}, with the same
			computational budget of $T=5\times10^9$ and 
			$\tau_1(n_0) = \min (500,\sqrt{n_0})$.  
			Note that the results for ONMC and the reference approach overlap.
			\label{fig:hists}		\vspace{-28pt}}
\end{figure}	

\cite{plummer2015cuts} noticed that WinBugs and OpenBugs~\citep{spiegelhalter1996bugs} similarly
 do not provide valid inference when using their cut function primitives, which effectively
 allow the definition of nested inference problems.  However,
they do not notice the equivalence to the NMC  formulation and instead propose
a heuristic for reducing the bias that itself has no theoretical guarantees.

\vspace{2pt}

\section{NESTED CONDITIONING}
\label{sec:cond}

An alternative way one might wish to nest queries is to 
use the partition function estimate of one query to factor the trace probability of another.
We refer to this as \emph{nested conditioning}.  
In its simplest form, we can think about conditioning on the values input
to the inner query.  In Anglican we can carry this out by
using the following custom distribution object constructor
\vspace{-4pt}
\begin{lstlisting}[basicstyle=\ttfamily\footnotesize,frame=none]
(defdist nest [inner inputs inf-alg M] []
 (sample [this] nil)
 (observe [this _] 
  (log-marginal (take M 
   (doquery inf-alg inner inputs)))))    
\end{lstlisting}
\vspace{-8pt}
When the resulting distribution object is observed, this will now generate, and factor the
trace probability by,
a partition function estimate for \clj{inner} with inputs \clj{inputs}, constructed using \clj{M} 
samples of the inference algorithm \clj{inf-alg}.  For example, if we were to use the query 
\vspace{-4pt}
\begin{lstlisting}[basicstyle=\ttfamily\footnotesize,frame=none]
(defquery outer [D]
 (let [y (sample (beta 2 3))]
  (observe (nest inner [y D] :smc 100) nil) 
  y))
\end{lstlisting}
\vspace{-8pt}
with \clj{inner} from the nested inference example, then this would form a
pseudo marginal sampler~\citep{andrieu2009pseudo} for the unnormalized target distribution
\begin{align*}
\pi_c(y,D)=&\textsc{Beta}(y;2,3) \int  \, \Gamma(z;y,1) \, \mathcal{N}(D;y,z^2) dz.
\end{align*}

Unlike the nested inference case, nested conditioning turns out to be valid even if our
budget is held fixed, provided that the partition function estimate is unbiased,
as is satisfied by, for example, importance sampling and SMC.  In fact, it is important to hold the budget fixed to achieve 
a MC convergence rate. In general, we can define our
target density as
\begin{align}
\label{eq:obs-target}
p_o(x,y) \propto \pi_o(x,y) = \psi(x,y) p_i(y),
\end{align}
where $\psi(x,y)$ is as before (except that we no longer have returned
variables from the inner query) and $p_i(y)$ is the true partition function
of the inner query when given input $y$.  In practice, we cannot evaluate
$p_i(y)$ exactly, but instead produce unbiased estimates $\hat{p}_i(y)$.
Using an analogous self-normalized importance sampling to the nested inference
case leads to the weights
\begin{align}
w_{n} = 
\psi(x_n,y_n) \hat{p}_i (y_n)/q (x_n,y_n)
\label{eq:weights-obs}
\end{align}
and corresponding empirical measure
\begin{align}
\hat{p}(\cdot) = \frac{1}{\sum_{n=1}^{N_0}w_{n}}
\sum_{n=1}^{N_0}	w_{n,} \delta_{(x_n,y_n)}(\cdot)
\label{eq:emp-measure-obs}
\end{align}
such that we are conducting conventional MC estimation, but our
weights are now themselves random variables for a given $(x_n,y_n)$ due to
the $\hat{p}_i (y_n)$ term.  However, the weights are unbiased estimates
of the ``true weights'' $\psi(x_n,y_n) p_i (y_n)/q (x_n,y_n)$ such 
that we have proper weighting~\citep{naessethLS2015nested} and thus
convergence at the standard MC rate, provided the budget of the
inner query remains fixed.  This result also follows directly from
Theorem 6 of~\cite{rainforth2017pitfalls}, which further ensures no complications arise when
conditioning on multiple queries if the corresponding partition function
estimates are generated independently.
These results further trivially extend to the repeated nesting case by recursion, while
using the idea of pseudo-marginal methods~\citep{andrieu2009pseudo}, the results
also extend to using MCMC based inference for the outermost query.

%

Rather than just fixing the inputs to the nested query, one can also consider
conditioning on the internally sampled variables in the program taking on certain values.  
Such a nested conditioning approach has been implicitly carried out 
by~\citet{rainforth2016bayesian,zinkov2016composing,scibior2016modular,turing17},
each of which manipulate the original program in some fashion to construct a partition
function estimator that is used used within a greater inference scheme, e.g. a PMMH estimator~\citep{andrieu2010particle}.

\section{ESTIMATES AS VARIABLES}
\label{sec:var}

Our final case is that one might wish to use estimates as first class variables in another query.  
In other words, a variable in an outer query is assigned to a MC expectation estimate calculated
from the outputs of running inference on another, nested, query.  By comparison, the nested inference
case (without Rao-Blackwellization) can be thought of as assigning a variable in the outer query 
to a single approximate sample from the conditional distribution of the inner query, rather than an MC expectation
estimate constructed by averaging over multiple samples.

Whereas nested inference can only encode a certain class
of nested estimation problems -- because the only nonlinearity originates from taking
the reciprocal of the partition function -- using estimates as variables allows, in principle, 
the encoding of any nested estimation.  This is because using the estimate as a first class
variable allows arbitrary nonlinear mappings to be applied by the outer query.  

An example of this approach is shown in Appendix~\ref{sec:exp}, where we construct a generic estimator
for Bayesian experimental design problems.  Here a partition function estimate is constructed
for an inner query and is then used in an outer query.  The output of the outer query depends on the
logarithm of this estimate, thereby creating the nonlinearity required to form a nested expectation.

Because using estimates as variables allows the encoding of any nested estimation problem, 
the validity of doing so is equivalent
to that of NMC more generally and must thus satisfy the requirements set out in~\citep{rainforth2017pitfalls}.  
In particular, one needs to ensure that the budgets used for
the inner estimates increase as more samples of the outermost query are taken.

%
%
%

\section{ONLINE NESTED MONTE CARLO}
\label{sec:ONMC}

NMC will  be highly inconvenient to actually implement in 
a PPS whenever one desires to provide online estimates; for example,
a lazy sequence of samples that converges to the target distribution.
Suppose that we have already calculated an NMC estimate, but now desire to 
refine it further.
In general, this will require an increase to all $N_k$ for each 
sample of the outermost estimator.  Consequently, the previous samples of the
outermost query must be revisited to refine their estimates.
This significantly complicates practical implementation, necessitating additional communication between
queries, introducing computational overhead, and potentially substantially increasing the memory requirements.

To highlight these shortfalls concretely, consider
the nested inference class of problems and, in particular,
constructing the un--Rao--Blackwellized
estimator~\eqref{eq:emp-measure-nonRB} in an online fashion. Increasing
$N_1$ requires $m^*(n)$ to be redrawn for each $n$, which in
turn necessitates storage of previous samples and weights.\footnote{Note that
not all previous samples and weights need storing -- when making the update
we can sample whether to change $m^*(n)$ or not based on combined weights
from all the old samples compared to all the new samples.}  This leads to an overhead cost
from the extra computation carried out for re-visitation
and a memory overhead from having to store information about each call
of the inner query.  

Perhaps even more problematically, 
the need to revisit old samples when drawing new samples can cause
substantial complications for implementation.  Consider
implementing such an approach in Anglican.  Anglican is designed to 
return a lazy infinite sequence of samples converging to the target
distribution.  Once samples are taken from this sequence, they become
external to Anglican and cannot be posthumously updated when
further samples are requested.  Even when all the output samples
remain internal, revisiting samples remains difficult: one either needs
to implement some form of memory for nested queries so they can be
run further, or, if all information is instead stored at the outermost level,
additional non-trivial code is necessary to apply post-processing and to revisit
queries with previously tested inputs.  The latter of these is likely to
necessitate inference--algorithm--specific changes, particularly when there
are multiple levels of nesting, thereby hampering the entire language construction.

To alleviate these issues, we propose to only increase the computational
budget of \emph{new} calls to nested queries, such that earlier calls use fewer samples
than later calls.  This simple adjustment removes the need for
communication between different calls and requires only the 
storage of the number of times the outermost query has previously been sampled
to make updates to the overall estimate.
We refer to this approach as \emph{online} NMC (ONMC),
which, to the best of our knowledge, has not been previously considered in the literature.
As we now show, ONMC only leads to small
changes in the convergence rate of the resultant estimator compared 
to NMC: using recommended parameter settings, the asymptotic root mean squared error
for ONMC is never more than twice that of NMC
for a matched sample budget and can even be smaller.

Let $\tau_k (n_0) \in \mathbb{N}^+, k=1,\dots,D$ be monotonically increasing
functions dictating the number
of samples used by ONMC at depth $k$ for the $n_0$-th iteration of
the outermost estimator.  The ONMC estimator is defined as
\begin{align}
	\label{eq:onmc}
J_0 &
= \frac{1}{N_0}
\sum_{n_0=1}^{N_0} f_0\left(y^{(0)}_{n_0}, I_{1}\left(y^{(0)}_{n_0}, \tau_{1:D}(n_0)\right)\right) 
\end{align}
where $I_{1}(y^{(0)}_{n_0}, \tau_{1:D}(n_0))$ is calculated using $I_{1}$ in
\eqref{eq:nmc}, setting $y^{(0)}=y^{(0)}_{n_0}$ and $N_k = \tau_k(n_0), \forall k \in 1, \dots, D$.  
For reference, the NMC estimator, $I_0$, is as per~\eqref{eq:onmc}, except for replacing $\tau_{1:D}(n_0)$ with
$\tau_{1:D}(N_0)$.  Algorithmically, we have that the ONMC approach is defined as follows.
\vspace{-5pt}
\begin{algorithm}[h]
	\small
	\caption{Online Nested Monte Carlo \label{alg:onmc}}
	\setstretch{1}
	\begin{algorithmic}[1]
		\State $n_0\leftarrow 0, \quad J_0 \leftarrow 0$
		\While{true}
		\State $n_0\leftarrow n_0+1, \quad y_{n_0}^{(0)} \sim p(y^{(0)})$
		\State Construct $I_{1}\left(y^{(0)}_{n_0}, \tau_{1:D}(n_0)\right)$
		using $N_k = \tau_k(n_0)\, \forall k$
		\State $J_0 \leftarrow \frac{n_0-1}{n_0}J_0 + 
		f_0\big(y^{(0)}_{n_0}, I_{1}\big(y^{(0)}_{n_0}, \tau_{1:D}(n_0)\big)\big)$
		\EndWhile
	\end{algorithmic}
\end{algorithm}
\vspace{-10pt}

We see that OMMC uses fewer samples at inner layers for earlier samples of
the outermost level,
and that each of resulting inner estimates is calculated as per an NMC estimator with a reduced sample budget.  
We now show the consistency of the ONMC estimator.  
\begin{restatable}{theorem}{conv}
	\label{the:conv}
	If each $\tau_k(n_0) \ge A \left(\log(n_0)\right)^\alpha, \forall n_0 > B$ for some constants $A,B,\alpha>0$ 
	and each $f_k$ is continuously differentiable,
	then the mean
	squared error of $J_0$ as an estimator for $\gamma_0$ converges to zero as $N_0\to\infty$.
\end{restatable}
In other words, ONMC converges for any realistic choice of $\tau_k(n_0)$ provided
$\lim_{n_0\to\infty} \tau_k(n_0) = \infty$: the requirements on $\tau_k(n_0)$
are, for example, much weaker than requiring a logarithmic or faster
rate of growth, which would already be an impractically slow rate of increase.

In the case where $\tau_k(n_0)$ increases at a
polynomial rate, we can further quantify the rate of convergence, along with the relative
variance and bias compared to NMC:
\begin{restatable}{theorem}{rate}
	\label{the:rate}
	If each $\tau_k(n_0) \ge A n_0^{\alpha}, \, \forall n_0 > B$ for some constants 
	$A,B,\alpha > 0$ and each $f_k$ is continuously differentiable, then
	\begin{align}
	&\hspace{-4pt}\E \left[\left(J_0-\gamma_0\right)^2\right] \le
	\frac{\varsigma_0^2}{N_0}+
	\left(\frac{\beta g(\alpha,N_0)}{A N_0^\alpha}\right)^2+O(\epsilon),
	\label{eq:onmc-rate} \displaybreak[0]\\
	&\text{where} \quad g(\alpha,N_0) = \begin{cases}
	1/(1-\alpha),  &\alpha<1 \\
	\log (N_0)+\eta,  &\alpha=1 \\
	\zeta(\alpha) N_0^{\alpha-1}, &\alpha>1
	\end{cases}; \displaybreak[0]\\
		&\quad \quad \beta = \frac{C_0 \varsigma_{1}^2}{2}
		+\sum_{k=0}^{D-2}  \left(\prod_{d=0}^{k} K_{d}\right)
		\frac{C_{k+1} \varsigma^2_{k+2}}{2}; \label{eq:beta}
	\end{align}
	$\eta\approx 0.577$ is the Euler--Mascheroni constant;
	$\zeta$ is the Riemann--zeta function; and
	$C_k$, $K_k$, and $\varsigma_k$ are constants defined
	as per the corresponding NMC bound given in~\eqref{eq:bound-cont}.
\end{restatable}\vspace{-6pt}

\begin{restatable}{corollary}{relative}
	\label{the:relative-rate}
	Let $J_0$ be an ONMC estimator setup as per Theorem~\ref{the:rate} with $N_0$ outermost samples and
	let $I_0$ be an NMC estimator with a matched overall sample budget.
	Defining $c = (1+\alpha D)^{(-1/(1+\alpha D))}$, then
	\begin{align*}
	\var [J_0] &\to c \var [I_0] \quad \text{as} \quad N_0\to \infty.
	\end{align*}
	Further, if the NMC bias decreases at a rate proportional to that implied by the bound given
	in~\eqref{eq:bound-cont}, namely
	\begin{align}
	\left|\E [I_0-\gamma_0]\right| = \frac{b}{M_0^{\alpha}} +O(\epsilon)
	\label{eq:bias-assumpt}
	\end{align}
	for some constant
	$b>0$, where $M_0$ is the number of outermost samples used by the NMC sampler, then
	\begin{align*}
	\left|\E [J_0-\gamma_0]\right| &\le c^{\alpha} g(\alpha,N_0)\left|\E [I_0-\gamma_0]\right| +O(\epsilon).
	\end{align*}
\end{restatable}
We expect the assumption that the bias scales 
as $1/M_0^{\alpha}$ to be satisfied in the vast majority of scenarios, but there
may be edge cases, e.g. when an $f_k$ gives a constant output, 
for which faster rates are observed.
Critically, the assumption holds for all nested inference problems
because the rate given in~\eqref{eq:nest-inf-rate} is an equality.

We see that if $\alpha<1$, which will generally be the case in
practice for sensible setups, then the convergence rates for
ONMC and NMC vary only by a constant factor.  Specifically, for a fixed
value of $N_0$, they have the same asymptotic variance and ONMC has
a factor of $1/(1-\alpha)$ higher bias.  However, the cost of ONMC
is (asymptotically) only $c<1$ times that of NMC, so for a fixed overall
sample budget it has lower variance.  

As the bound varies only
in constant factors for $\alpha<1$, the asymptotically optimal value for $\alpha$ for ONMC
is the same as that for NMC, namely $\alpha=0.5$~\citep{rainforth2017pitfalls}.  
For this setup, we have $c\in
\{0.763,0.707,0.693,0.693,0.699,1\}$ respectively for $D \in \{1,2,3,4,5,
\infty\}$.  Consequently, when $\alpha=0.5$, the 
fixed budget variance of ONMC is always better than NMC, while the
bias is no more than $1.75$ times larger if
$D\le13$ and no more than $2$ times large more generally.

\subsection{EMPIRICAL CONFIRMATION}

To test ONMC empirically, we consider the simple analytic model given in Appendix~\ref{sec:simple-model},
setting $\tau_1 (n_0) = \max(25,\sqrt{n_o})$. The rationale for setting a minimum value of $N_1$ is to minimize the burn-in effect of ONMC -- earlier samples will have larger
bias than later samples
and we can mitigate this by ensuring a minimum value for $N_1$.  More generally, we recommend setting
(in the absence of other information)
$\tau_1(n_0) = \tau_2(n_0) = \dots = \tau_D (n_0) = \max(T_{\min}^{1/3},\sqrt{n_0})$, where $T_{\min}$ is
the minimum overall budget we expect to spend.  In Figure~\ref{fig:onmc}, we have chosen to set $T_{\min}$
deliberately low so as to emphasize the differences between NMC and ONMC.  Given our value for $T_{\min}$,
the ONMC approach is identical to fixing $N_1 = 25$ for $T<25^3=15625$, but unlike fixing $N_1$, it continues to
improve beyond  this because it is not limited by asymptotic bias.  Instead, we see an inflection point-like
behavior around $T_{\min}$, with the rate recovering to effectively match that of the NMC estimator.  

\begin{figure}[t]
	\centering
	{\includegraphics[width=0.95\textwidth,trim={0.5cm 0 3.5cm 2cm},clip]{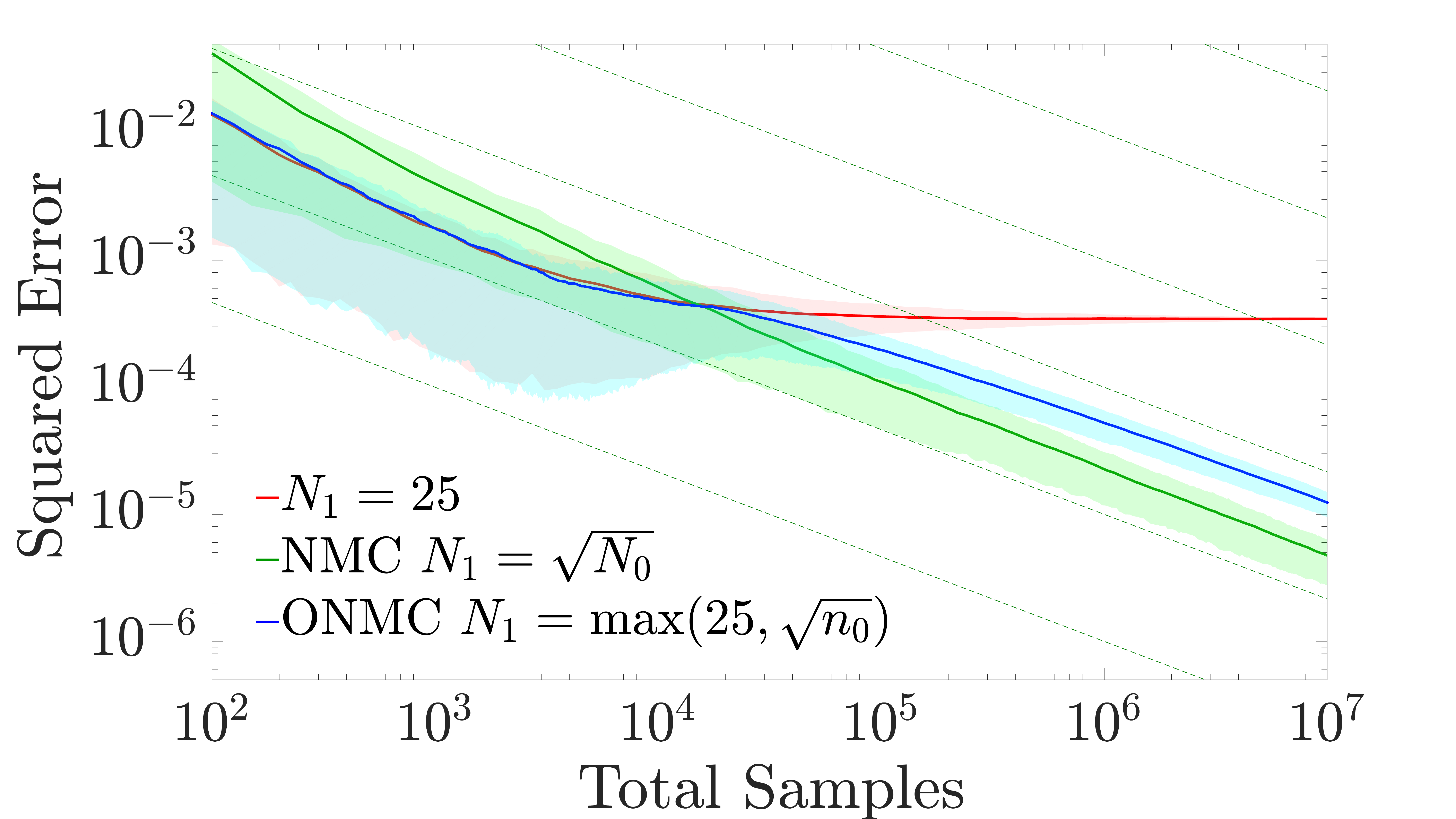}}
	\vspace{-8pt}
	\caption{Convergence of ONMC, NMC, and fixed $N_1$.  Results are averaged over $1000$ runs, with 
		solid lines showing the mean and shading the 25-75\% quantiles.  The theoretical
		rates for NMC are shown by the dashed lines.
		\label{fig:onmc}		\vspace{-8pt}}
\end{figure}

\subsection{USING ONMC IN PPSs}

Using ONMC based estimation schemes to ensure consistent estimation for nested inference
in PPSs is straightforward -- the number of iterations
the outermost query has been run for is stored and used 
to set the number of iterations used for
the inner queries.  In fact, even this minimal level of communication is
not necessary -- $n_0$
can be inferred from the number of times we have previously run inference on the current query, the current depth $k$,
and $\tau_{1}(\cdot),\dots,\tau_{k-1}(\cdot)$. 

As with NMC, for nested inference problems ONMC can either return a single sample
from each call of a nested query, or Rao--Blackwellize the drawing of this sample when
possible.  Each respectively produces an estimator analogous to~\eqref{eq:emp-measure-nonRB} 
and~\eqref{eq:emp-measure} respectively, except that $N_1$ in the definition of the
inner weights is now a function of $n$.

%

Returning to Figure~\ref{fig:hists}, we see that using ONMC with nested importance sampling
and only returning a single sample corrects the previous issues with how Anglican deals
with nested inference, producing samples indistinguishable from the reference code.


\section{CONCLUSIONS}


We have formalized the notion of nesting probabilistic program queries and investigated
the statistical validity of different categories of nesting.  We have found that current
systems tend to use methods that lead to asymptotic bias for nested inference problems, but that
they are consistent for nested conditioning.  We have shown how to carry out the former
in a consistent manner
and developed a new \emph{online} estimator that simplifies the construction algorithms
that satisfy the conditions required  for convergence.
%

\newpage
\appendix


\section{PROOFS}
\label{sec:proofs}


\nestinf*
\begin{proof}
Though informally the high-level result follows directly from \citep[Theorem 3]{rainforth2017pitfalls},
there are three subtleties that require further attention.
Firstly, unlike~\citep[Theorem 3]{rainforth2017pitfalls}, this result is an asymptotic equality
rather than a bound -- in the limit of large $N_0, N_1$ it holds exactly.  
This more powerful result is made possible by knowing the
exact form of the nonlinearity.
Secondly, our overall estimator uses the ratio of two NMC estimators.
Though Slutsky's Theorem means this does not create complications in the general 
demonstration of convergence, additional care is required when calculating the exact rate.
Finally, samples are reused in both the inner and outer estimators.
This could easily be avoided by sampling an additional $z$ for the outer estimator, thereby giving
an estimator trivially of the form considered by \citep[Theorem 3]{rainforth2017pitfalls}.  However, doing so
would be less efficient and is expected to have a larger variance
than the estimator used.

We start by considering the the partition function estimate, noting that true value is
$Z=\iiint \pi_o (x,y,z) dxdydz$,
\begin{align}
	\hat{Z} &= \frac{1}{N_0}\sum_{n=0}^{N_0}
	\cfrac{\cfrac{1}{N_1}\sum\limits_{m=1}^{N_1}
		\cfrac{\psi(x_n,y_n,z_{n,m}) \pi_i (y_n,z_{n,m})}{q (x_n,y_n,z_{n,m})}}
		{\cfrac{1}{N_1}\sum\limits_{m=1}^{N_1} \cfrac{\pi_i (y_n,z_{n,m})}{q (z_{n,m}|y_n)}}	\label{eq:ml} \displaybreak[0] \\
		&=
	\frac{1}{N_0}\sum_{n=0}^{N_0}
	\frac{\frac{1}{N_1}\sum_{m=1}^{N_1} v_{n,m}}
	{\frac{1}{N_1}\sum_{m=1}^{N_1} u_{n,m}} \displaybreak[0] 
\end{align}
where 
\begin{align}
\label{eq:u}
u_{n,m} &= \frac{\psi(x_n,y_n,z_{n,m}) \pi_i (y_n,z_{n,m})}{q (x_n,y_n,z_{n,m})} \quad \text{and}\displaybreak[0]\\
\label{eq:v}
v_{n,m} &= \frac{\pi_i (y_n,z_{n,m})}{q (z_{n,m}|y_n)}
\end{align}
will be used as shorthands.  Further
defining 
\begin{align}
\label{eq:piy}
\pi_i(y_n) = \int &\pi_i(y_n,z) d z, \displaybreak[0]\\
\label{eq:U}
V_n = \frac{1}{N_1}\sum_{m=1}^{N_1} v_{n,m}, \,\,\,
&\text{and} \,\,\, U_n = \frac{1}{N_1}\sum_{m=1}^{N_1} u_{n,m}, \displaybreak[0]
\end{align}
and using Taylor's Theorem on $1/U_n$ about $\pi_i(y_n)$ gives
\begin{align}
\begin{split}
&\hat{Z} = O(\epsilon)+	\frac{1}{N_0}\sum_{n=0}^{N_0}
\frac{V_n}{\pi_i(y_n)} \\
&\quad\times \left(1+\frac{\pi_i(y_n)-U_n}{\pi_i(y_n)}+
\left(\frac{\pi_i(y_n)-U_n}{\pi_i(y_n)}\right)^2\right)
\label{eq:z-expan}
\end{split}
\end{align}	
provided each $U_n, \pi(y_n) \neq 0$ to avoid singularity issues.  
We have by assumption that 
$\pi(y_n)\neq0$ for all possible $y_n$ as otherwise the problem becomes
ill-defined.  On the other hand, if $U_n=0$, it must also
be the case that $V_n=0$.  Here by taking the convention $V_n/U_n =0$ when $U_n=V_n=0$,
we can avoid all further possible singularity issues, such
that~\eqref{eq:z-expan} always holds.

Meanwhile, the standard breakdown
of the mean squared error to the variance plus the bias squared gives
\begin{align*}
\E \left[\left(\hat{Z}-Z\right)^2\right] =
\var \left[\hat{Z}\right] +\left(\E \left[\hat{Z}-Z\right]\right)^2.
\end{align*}
Using~\eqref{eq:z-expan}, we see that the first term in
the expansion dominates for the variance (as $\pi_i(y_n)-U_n$ decreases with $N_1$), such that the
weak law of large numbers gives
\begin{align*}
\var \left[\hat{Z}\right]  = \frac{1}{N_0} \var \left[
\frac{V_1}{\pi_i(y_1)}\right] + O(\epsilon).
\end{align*} 
Now we have
\begin{align*}
V_1
&= \E [v_{1,1} | x_1, y_1] +\frac{1}{N_1}\sum_{m=1}^{N_1} \left(v_{1,m}-
\E [v_{1,m} | x_1, y_1]\right)
\end{align*}
 and we further see from the weak law of large numbers
 that the second term tends to $0$ as $N_1$ increases, but
 the first term remains fixed.  Thus the first term is
 dominant and we have
 \begin{align}
&\var \left[\hat{Z}\right] =
\frac{1}{N_0} \var \left[
\frac{\E [v_{1,1} | x_1, y_1] }{\pi_i(y_1)}\right] + O(\epsilon) 
\displaybreak[0]\\
&=
\frac{1}{N_0} \var \left[
\frac{
		\int \psi(x_1,y_1,z) \pi_i (y_1,z) dz
	}{q(x_1,y_1)\pi_i(y_1)}\right] + O(\epsilon) \displaybreak[0]\\
&=
\frac{\sigma_z^2}{N_0}  + O(\epsilon) \displaybreak[0]
 \end{align}
 where
 \begin{align}
 \label{eq:sigmaz}
 \sigma_z^2 = \var \left[\cfrac{
 	\int 
 	\pi_o(x_1,y_1,z) dz
 }{q(x_1,y_1)}\right].
 \end{align}

Switching focus to the bias we have
\begin{align*}
&\E \left[\hat{Z}-Z\right]
=  O(\epsilon)+ \E \bigg[
\left( \frac{V_1}{\pi_i(y_1)}\right)\\
&\quad\quad\quad\times \left(\frac{\pi_i(y_1)-U_1}{\pi_i(y_1)}+
\left(\frac{\pi_i(y_1)-U_1}{\pi_i(y_1)}\right)^2\right)\bigg]  \\
&\quad= O(\epsilon)+ \E \bigg[ \E \bigg[
\left( \frac{v_{1,1}}{\pi_i(y_1)}\right)\\
&\quad\quad\quad\times \left(\frac{\pi_i(y_1)-U_1}{\pi_i(y_1)}+
\left(\frac{\pi_i(y_1)-U_1}{\pi_i(y_1)}\right)^2\right) \bigg| 
y_1 \bigg] \bigg].
\end{align*}
 For the first order term in the expansion, only the component
 with respect to $u_{1,1}$ is non-zero as, for $m\neq1$, 
\begin{align}
\label{eq:bias_zero}
\begin{split}
\E [v_{1,1} & \left(\pi_i(y_1)-u_{1,m}\right) | y_1]= \\
&\E [v_{1,1} | y_1]
\E [\left(\pi_i(y_1)-u_{1,m}\right) | y_1] = 0.
\end{split}
\end{align}
Denoting the first order term as $T_1$, we thus have
\begin{align*}
& T_1 = \E \left[\frac{v_{1,1} \left(\frac{1}{N_1} \sum_{m=1}^{N_1} \pi_i(y_1)- u_{1,m}\right)}{
	\left(\pi_i(y_1)\right)^2}\right]\\
 &=\frac{1}{N_1} \left(\E \left[\frac{v_{1,1}}{\pi_i(y_1)}\right] - \E \left[\frac{v_{1,1}u_{1,1}}{
	\left(\pi_i(y_1)\right)^2}\right] \right)\\
&\hspace{-3pt}= \frac{1}{N_1}
\left(Z-\iiint \frac{\psi(x,y,z)\left(\pi_i(y,z)\right)^2}
{q(z|y) \left(\int \pi_i(y,z') dz'\right)^2} dxdydz\right) \\
&\hspace{-3pt}= \frac{1}{N_1}
\left(Z-\iiint \frac{\pi_o(x,y,z)p_i(z|y)}{q(z|y)} dxdydz\right).
\end{align*}
For the second order term, $T_2$, components of $u_{1,m}$ for $m\neq1$ are
no longer zero as follows
\begin{align*}
&T_2 = \E \left[ \E \left[\frac{v_{1,1}}{\pi_i(y_1)}
 \left(\frac{1}{N_1}\sum_{m=1}^{N_1} \frac{\pi_i(y_1)-u_{1,m}}{\pi_i(y_1)}\ \right)^2
 \middle|  y_1 \right] \right] \displaybreak[0] \\
 &= \frac{1}{N_1^2}
 \E \left[\frac{v_{1,1} \left(\pi_i(y_1)-u_{1,1}\right)^2}{
 	\left(\pi_i(y_1)\right)^3}\right] +  \frac{1}{N_1^2}\E \Bigg[\frac{1}{\left(\pi_i(y_1)\right)^3} \times\\
& \E \Bigg[v_{1,1} \sum_{m=2}^{N_1}
 \sum_{\ell=1}^{N_1}\left(
 \pi_i(y_1)-u_{1,m} \right)\left(
 \pi_i(y_1)-u_{1,m} \right)\Bigg| y_1 \Bigg]\Bigg],\displaybreak[0] \\
 \intertext{now using an argument akin to~\eqref{eq:bias_zero} shows that
 	terms for which $m\neq\ell$ are all zero.  Further noticing that the first term
 	is asympotitically dominated gives}
  &= O(\epsilon)+  \\
  &\,\,\,\, \frac{1}{N_1^2}\E \Bigg[\frac{1}{\left(\pi_i(y_1)\right)^3}\E \Bigg[v_{1,1} \sum_{m=2}^{N_1}\left(
  \pi_i(y_1)-u_{1,m} \right)^2\Bigg| y_1 \Bigg]\Bigg],\displaybreak[0] \\
  &= O(\epsilon)+  \left(\frac{N_1-1}{N_1^2}\right) \times \\
  &\,\,\,\, \E \left[\E \left[\frac{v_{1,1}}{\pi_i(y_1)}\middle| y_1 \right] \E \left[ \left(
  \frac{\pi_i(y_1)-u_{1,1}}{\pi_i(y_1)} \right)^2\middle| y_1 \right]\right],\displaybreak[0] \\  
  &= O(\epsilon)+\\
   &\,\,\,\, \frac{1}{N_1} \E \left[\cfrac{
  	\int 
  	\pi_o(x,y_1,z) dxdz
  }{q(y_1)} \, \var \left[\frac{u_{1,1}}{\pi_i(y_1)} \middle| y_1\right] \right],\displaybreak[0] \\  
  &= O(\epsilon)+\\
  &\,\,\,\, \frac{1}{N_1} \E \left[\cfrac{
  	\iint 
  	\pi_o(x,y_1,z) dxdz
  }{q(y_1)} \, \var \left[\frac{p_i(z_{1,1}|y_1)}{q(z_{1,1}|y_1)} \middle| y_1\right] \right],\displaybreak[0] 
\end{align*}
Putting the bias terms together now gives
\begin{align}
\E \left[\hat{Z}-Z\right]
= \frac{\delta_z}{N_1} + O(\epsilon)
\end{align}
where
\begin{align}
\delta_z =& \iiint
	\pi_o(x,y,z)
 \var \left[\frac{p_i(z_{1,1}|y_1)}{q(z_{1,1}|y_1)} \middle| y_1=y\right] dxdydz \nonumber \\
&+Z-\iiint \frac{\pi_o(x,y,z)p_i(z|y)}{q(z|y)} dxdydz. \label{eq:deltaz}
\end{align}

We thus have that the mean squared error is
\begin{align}
\E \left[\left(\hat{Z}-Z\right)^2\right] =
\frac{\sigma_z^2}{N_0} + \frac{\delta_z^2}{N_1^2}+O(\epsilon)
\end{align}
where $\sigma_z^2$ and $\delta_z$  a respectively defined in~\eqref{eq:sigmaz}
and~\eqref{eq:deltaz}.
 
If we now consider the estimator for the unnormalized target expectation
(i.e. the numerator in the self-normalized estimator), we see that
we can use the same arguments with $\psi(x,y,z)$ replaced
by $\psi(x,y,z) g(x,y,z)$.  Thus denoting this estimator as $\hat{G}$ and
its true value as $G=\gamma_0 Z$, we have
\begin{align}
\E \left[\left(\hat{G}-G\right)^2\right] =
\frac{\sigma_g^2}{N_0} + \frac{\delta_g^2}{N_1^2}+O(\epsilon)
\end{align}
where
\begin{align}
\label{eq:sigmag}
\sigma_g &= 
 \var \left[\cfrac{
	\int g(x_1,y_1, z)
	\pi_o(x_1,y_1,z) dz
}{q(x_1,y_1)}\right] \displaybreak[0] \\
\label{eq:deltag}
\begin{split}
\delta_g &= G-\iiint \frac{g(x,y,z) \pi_o(x,y,z)p_i(z|y)}{q(z|y)} dxdydz+\\
&
\iiint \pi_o(x,y,z) \var \left[\frac{p_i(z_{1,1}|y_1)}{q(z_{1,1}|y_1)} \middle| y_1=y\right] dxdydz.
\end{split}
\end{align}

Now the self-normalized estimator we actually use is
 $I_0 = \hat{G}/\hat{Z}$.  To assess this, we represent
 \begin{align}
 \hat{G} &= G+\frac{\delta_g}{N_1}+\frac{\sigma_g \xi_1}{\sqrt{N_0}} +O(\epsilon)\\
 \hat{Z} &= Z+\frac{\delta_z}{N_1}+\frac{\sigma_z \xi_2}{\sqrt{N_0}} +O(\epsilon)
 \end{align}
 where $\xi_1$ and $\xi_2$ are correlated random variables, each with
 mean zero and variance $1$ under their marginal distributions.  Now
 again using Taylor's theorem
\begin{align}
\frac{1}{\hat{Z}} &= \frac{1}{Z}\left(1+\frac{Z-\hat{Z}}{Z}\right)+ O(\epsilon)\nonumber \\
&= \frac{1}{Z}\left(1-\frac{1}{Z}\left(\frac{\delta_z}{N_1}+
\frac{\sigma_z \xi_2}{\sqrt{N_1}} \right)\right)+ O(\epsilon)
\end{align} 
where singularity issues are again dealt with because $Z\neq0$ by
our assumptions, noting $\hat{G}=0$ if $\hat{Z}=0$, and taking
the convention $\hat{G}/\hat{Z}=0$ whenever $\hat{G}=0$.  Thus
\begin{align*}
&I_0 =\frac{1}{Z^2}
\left(G+\frac{\delta_g}{N_1}+\frac{\sigma_g \xi_1}{\sqrt{N_0}}\right)
\hspace{-3pt}
 \left(Z-\frac{\delta_z}{N_1} - \frac{\sigma_z \xi_2}{\sqrt{N_0}}\right)
\hspace{-2pt}+\hspace{-2pt}O(\epsilon) \\
&=\gamma_0+\frac{\delta_g-\gamma_0\delta_z}{Z N_1}
+\frac{\sigma_g \xi_1-\gamma_0\sigma_z\xi_2}{Z \sqrt{N_0}}-
\frac{\sigma_g \xi_1 \sigma_z\xi_2}{Z^2 N_0} \hspace{-2pt}+\hspace{-2pt} O(\epsilon)
\end{align*}
Therefore,
\begin{align}
\var \left[I_0\right] = \cfrac{\sigma_g^2
+  \gamma_0^2\sigma_z^2 -2\sigma_g \gamma_0 \sigma_z \text{Cov}(\xi_1,\xi_2)}{Z^2 N_0}
+ O(\epsilon)
\end{align}
and
\begin{align}
\E \left[I_0-\gamma_0\right] =
\frac{\delta_g-\gamma_0\delta_z}{Z N_1}
-\frac{\sigma_g \sigma_z}{Z^2 N_0} \E \left[\xi_1 \xi_2\right]+ O(\epsilon)
\end{align}
and therefore
\begin{align}
\E &\left[\left(I_0 - \gamma_0\right)^2\right] = \frac{\sigma^2}{N_0}
+\frac{\delta^2}{N_1^2} + O(\epsilon)
\end{align}
where
\begin{align}
\sigma^2 &= \frac{\sigma_g^2
	+  \gamma_0^2\sigma_z^2 -2\sigma_g \gamma_0 \sigma_z \text{Cov}(\xi_1,\xi_2)}{Z^2} 
\label{eq:sigma-delta} \\
\text{and} \quad \delta &= \frac{\delta_g-\gamma_0\delta_z}{Z}.
\label{eq:delta}
\end{align}
A full characterization of $\text{Cov}(\xi_1,\xi_2)$ can further be calculated by considering
the full expansions for $\hat{G}$ and $\hat{Z}$.  Though we do not trawl through the
necessary algebra here, we note that $\text{Corr}(\xi_1,\xi_2) = 1$ if $g(x,y,z)$ is
constant, in which case we also have $\delta_g=\gamma_0 \delta$ and $\sigma_g^2 = \gamma_0^2 \sigma_z^2$
and so $\delta=\sigma^2=0$.  This is to be expected as, in this scenario, we have the
trivial estimator $I_0 = \gamma_0 = g(x,y,z) \, \forall x,y,z$.
\end{proof}

\begin{corollary}
	The un-Rao-Blackwellized form of the estimator given 
	in~\eqref{eq:emp-measure-nonRB}, whereby only
	a single sample is returned from the inner query sampled in
	proportion to its weight, also converges.  Specifically, it has the
	same rate of convergence for the bias, but has a constant factor
	increase in the variance.
\end{corollary}
\begin{proof}
The un-Rao-Blackwellized estimator for the partition function
can be represented as
\begin{align}
\hat{Z}' &= \frac{1}{N_0}\sum_{n=0}^{N_0}
	\cfrac{\psi(x_n,y_n,z_{n,m^*(n)})}{q (x_n,y_n)}
\end{align}
where
\begin{align*}
m^*(n) \sim \textsc{Discrete}\left(\cfrac{
	\pi_i (y_n,z_{n,m})/q (z_{n,m}|y_n)}
{\sum_{\ell=1}^{N_1}\pi_i (y_n,z_{n,\ell})/q (z_{n,\ell}|y_n)}\right)
\end{align*}
We first show $\hat{Z}$ from~\eqref{eq:ml} is a true Rao-Blackwellization of
$\hat{Z}'$ by noting that
\begin{align*}
&\E \left[\hat{Z}' \middle| x_{1:N_0}, y_{1:N_0}, z_{1:N_0,1:N_1}\right] \\
&= \frac{1}{N_0}\sum_{n=0}^{N_0}\sum\limits_{m=1}^{N_1}
\cfrac{
	\cfrac{\psi(x_n,y_n,z_{n,m})}{q (x_n,y_n)}
\cfrac{\pi_i (y_n,z_{n,m})}{q (z_{n,m}|y_n)}}
{\sum\limits_{m=1}^{N_1} \cfrac{\pi_i (y_n,z_{n,m})}{q (z_{n,m}|y_n)}}
= \hat{Z}.
\end{align*}
We thus see that $\hat{Z}'$ and $\hat{Z}$ have the same expectation as
required.  Equivalent arguments can further be applied to show the unnormalized
target estimate has the same expectation as before.

For the variance, we can consider that
\begin{align*}
\var &\left[\hat{Z}'\right]=\frac{1}{N_0}
\var \left[\cfrac{\psi(x_1,y_1,z_{1,m^*(n)})}{q (x_1,y_1)}\right] \\
=&\frac{1}{N_0}
\var \Bigg[\cfrac{\psi(x_1,y_1,z^*)}{q (x_1,y_1)} \\
&\quad \quad \quad -
\cfrac{\psi(x_1,y_1,z^*)-\psi(x_1,y_1,z_{1,m^*(n)})}{q (x_1,y_1)}\Bigg]
\end{align*}
where $z^*\sim p_i(z|y)$.  Now as $N_1$ increases, the second of
these terms will diminish while the first does not, meaning the
first is dominant.

By following the same steps as Theorem~\ref{the:nestinf}, we thus 
achieve the same result for the convergence rate except for substituting
in for the following definitions
\begin{align}
\sigma_z^2 &= \var \left[\cfrac{\psi(x_1,y_1,z^*)}{q (x_1,y_1)}\right]  \\
\sigma_g^2 &= \var \left[\cfrac{g(x_1,y_1,z^*)\psi(x_1,y_1,z^*)}{q (x_1,y_1)}\right].
\end{align}
Note that these variances are always larger than those from
Theorem~\ref{the:nestinf}.
\end{proof}
 
\conv*
\begin{proof}
	Let $\hat{f}_{n_o} := f_0\left(y^{(0)}_{n_0}, I_{1}\left(y^{(0)}_{n_0}, \tau_{1:D}(n_0)\right)\right) $
	and let $I_0(n_0)$
	be a NMC estimator that uses $\tau_k(n_0)$ samples at each layer. 
	We have
	\begin{align}
	&~\E \left[\left(J_0-\gamma_0\right)^2\right] = \text{Var}\left[J_0\right]+\left(\E \left[J_0-\gamma_0\right]\right)^2 \nonumber \displaybreak[0]\\
	\hspace{-2pt} &= \frac{1}{N_0^2}\sum_{n_0=1}^{N_0} \text{Var}\left[\hat{f}_{n_o}\right] 
	+ \left(\frac{1}{N_0}\sum_{n_0=1}^{N_0} \E \left[\hat{f}_{n_o}-\gamma_0 \right]\right)^2 
	\nonumber \displaybreak[0] \\
	\hspace{-2pt} &=  \frac{1}{N_0^2}\sum_{n_0=1}^{N_0} \hspace{-2pt} n_0 \text{Var}\left[I_0(n_0)\right]\hspace{-2pt}
	+ \hspace{-2pt}
	\left(\frac{1}{N_0}\sum_{n_0=1}^{N_0} \E \left[I_0(n_0)\hspace{-2pt}-\hspace{-2pt}\gamma_0\right] \right)^2 \nonumber
	\end{align}
	Substituting in for the variance and bias terms from~\eqref{eq:bound-cont} now gives
	\begin{align}
	&\E \left[\left(J_0-\gamma_0\right)^2\right] \le O(\epsilon)+\frac{\varsigma_0^2}{N_0}+\label{eq:onmc-general}\\
	&\left(\frac{1}{N_0} \sum_{n_0=1}^{N_0} \left(
	\frac{C_0 \varsigma_{1}^2}{2 \tau_1(n_0)}
	+\sum_{k=0}^{D-2}  \left(\prod_{d=0}^{k} K_{d}\right)
	\frac{C_{k+1} \varsigma^2_{k+2}}{2 \tau_{k+2}(n_0)}\right)
	\right)^2\nonumber
	\end{align}
	Here $\varsigma_0^2/N_0$
	clearly tends to zero as $N_0\to\infty$.  
	For the bias squared term, which we denote $S(N_0)^2$, 
	we use the assumption that
	 $\tau_k(n_0) \ge A \left(\log(n_0)\right)^{\alpha}, \forall n_0 > B$.
	In the following analysis, we will assume that $\alpha<2$, 
	noting that if the result of the overall theorem holds for 
	$\alpha_1$, then it trivially holds for $\alpha_2>\alpha_1$.  We now have
	\begin{align*}
	& S(N_0)^2 \le \left(\frac{\lfloor B\rfloor}{N_0}S(\lfloor B\rfloor)+\frac{1}{N_0} \sum_{n_0=\lceil B\rceil}^{N_0} 
	\frac{\beta}{A \left(\log (n_0)\right)^{\alpha}}\right)^2 \\
	&\le 2 \left(\frac{\lfloor B\rfloor S(\lfloor B\rfloor)}{N_0}\right)^2 \hspace{-4pt}+ 2\left(\frac{1}{N_0}
	\sum_{n_0=\lceil B\rceil}^{N_0} \frac{\beta}{A \left(\log (n_0)\right)^{\alpha}}\right)^{2} \\
	&\le 2 \left(\frac{\lfloor B\rfloor S(\lfloor B\rfloor)}{N_0}\right)^2 \hspace{-4pt}+
	 \frac{2 \beta^2}{A^2 N_0^{\alpha}}
	\left(\sum_{n_0=\lceil B\rceil}^{N_0} \frac{1}{\left(\log (n_0)\right)^2}
	\right)^{\alpha}
	\end{align*} 
	where $\beta$ is as per~\eqref{eq:beta}.
	Here the first term clearly goes to zero because the assumption $\tau_k(n_0) \in \mathbb{N}^+$
	ensures $\lfloor B\rfloor S(\lfloor B\rfloor)$ is a finite constant.
	For the second term, we first note from using a condensation test that
	\begin{align}
	\sum_{n_0=\lceil B\rceil}^{N_0} \frac{1}{n_0\left(\log (n_0)\right)^2} < \infty.
	\end{align}
	Now by invoking Kronecker's lemma, namely that $\lim\limits_{N\to\infty} \frac{1}{N} \sum_{n}^{N} X_n =0$
	if $\sum_{n=1}^{\infty} X_n / n < \infty$, it follows that this term tends to
	zero.  Note that because we are examining the bound itself, rather than any 
	random variables, this is a result which holds surely.
	We have thus shown that all non-dominated terms in~\eqref{eq:onmc-general}
	tend to zero as $N_0\to\infty$, giving the required result.
\end{proof}\vspace{-8pt}

\rate*
\begin{proof}
	Starting at~\eqref{eq:onmc-general} and following on in the same manner as the proof for Theorem~\ref{the:conv},
	we have
	\begin{align*}
	S(N_0)^2 &\le \left(\frac{\lfloor B\rfloor}{N_0}S(\lfloor B\rfloor)+\frac{1}{N_0} \sum_{n_0=\lceil B\rceil}^{N_0} 
	\frac{\beta}{A n_0^{\alpha}}\right)^2 \\
	&=\left(\frac{\beta H_{\alpha}[N_0]}{A N_0}\right)^2+O(\epsilon)
	\end{align*}
	where $H_{\alpha}[N_0] := \sum_{n_0}^{N_0} n_0^{-\alpha}$ is the $N_0$-th generalized harmonic 
	number of order $\alpha$.  For $\alpha=1$ and $\alpha>1$, it is well known that
	$H_{1}[N_0] \to \log(N_0)+\eta$ and $H_{\alpha}[N_0] \to \zeta(\alpha)$ respectively.  For,
	$\alpha<1$, we apply the Euler-Maclaurin formula giving
	\begin{align*}
	H_{\alpha}[N_0] &= 1+ \int_{n_0=1}^{n_0=N_0} n_0^{-\alpha} dn_0+
	\frac{N_0^{-\alpha}-1}{2}+R_1 \\
	& \to N_0^{1-\alpha}/(1-\alpha).
	\end{align*}
	where the dominant term originates from the integral.
	Putting everything back together, namely substituting in turn for the bound on $S(N_0)^2$ and then
	this bound into~\eqref{eq:onmc-general}, now yields the desired result.
\end{proof}\vspace{-8pt}

\relative*
\begin{proof}
	We first consider how to match the sample budgets between the two estimators.  Noting that 
	asymptotically, the computational cost is dominated by calculations for the innermost
	estimator (see~\citep[Appendix G]{rainforth2017thesis}), we have for large $N_0$,
	\begin{align*}
	\text{Cost}_{\text{ONMC}} &\to \sum_{n_0}^{N_0} \prod_{k=1}^{D} \tau_k (n_0)
	= A^D \sum_{n_0}^{N_0} n_0^{\alpha D} \\
	&= A^D H_{-\alpha D} [N_0].
	\end{align*}
	The respective asymptotic cost for an NMC using $M_0$ outermost samples is
	\begin{align*}
	\text{Cost}_{\text{NMC}} &\to A^D M_0^{1+\alpha D}.
	\end{align*}
	Thus matching the computational budgets gives
	\begin{align}
	M_0 = \left(H_{-\alpha D} [N_0]\right)^{\frac{1}{1+\alpha D}}.
	\end{align}
	Now by applying the Euler-Maclaurin formula to $H_{-\alpha D} [N_0]$
	in similar manner to Theorem~\ref{the:rate}, we get
	\begin{align*}
	H_{-\alpha D} [N_0] \to \frac{N_0^{1+\alpha D}}{1+\alpha D}, \quad \text{and thus}
	\quad M_0 \to c N_0.
	\end{align*}
	Using~\eqref{eq:bound-cont} and Theorem~\ref{the:rate} we thus have	
	\begin{align*}
	\text{Var}\left[J_0\right] &\to \varsigma_0^2 / N_0 \to c \varsigma_0^2 / M_0 \to c\text{Var}\left[I_0\right].
	\end{align*}
	
	Now considering the biases, 
	\begin{align*}
	\left|\E [J_0-\gamma_0]\right| &= 
	\left|\frac{1}{N_0}\sum_{n_0=1}^{N_0} \E \left[I_0(n_0)\hspace{-2pt}-\hspace{-2pt}\gamma_0\right] \right| \\
	&\le \frac{1}{N_0}\sum_{n_0=1}^{N_0} \left| \E \left[I_0(n_0)\hspace{-2pt}-\hspace{-2pt}\gamma_0\right] \right| \\
	\intertext{and whenever~\eqref{eq:bias-assumpt} holds,}
	&=\frac{1}{N_0}\sum_{n_0=1}^{N_0} \frac{b}{n_0^{\alpha}} +O(\epsilon)	\\
	&=\frac{b H_{\alpha}[N_0]}{N_0} +O(\epsilon) \\
	&=\frac{b g(\alpha,N_0)}{N_0^{\alpha}} +O(\epsilon).
	\end{align*}
	By comparison,~\eqref{eq:bias-assumpt} also gives us
	\begin{align*}
	\left|\E [I_0-\gamma_0]\right| = \frac{b}{c^{\alpha} N_0^{\alpha}} +O(\epsilon)
	\end{align*}
	and so
	\begin{align*}
	\left|\E [J_0-\gamma_0]\right| &\le c^{\alpha} g(\alpha,N_0)\left|\E [I_0-\gamma_0]\right| +O(\epsilon)
	\end{align*}
	as required.
\end{proof}

\section{OBSERVING THE OUTPUT OF A NESTED QUERY}
\label{sec:app:nest-obv}

As discussed in Section~\ref{sec:samp}, one can construct nested inference problems
where one observes the output of, rather than sampling from, the nested query.  For
example, we could think about adjusting our previous example to the following
\vspace{-4pt}
\begin{lstlisting}[basicstyle=\ttfamily\small,frame=none]
(defquery inner [y D]
 (let [z (sample (gamma y 1))]
  (observe (normal y z) D)
  z))

(defquery outer [D]
 (let [y (sample (beta 2 3))
       x (sample (gamma 1 1))
       dist (conditional inner)]
  (observe (dist y D) x)
  (* y x)))
\end{lstlisting}
\vspace{-8pt}
Statistically, this problem is still well defined and can be represented
in the same form as~\eqref{eq:nest-inf}; Anglican's \sample and \observe
have the same impact on the distribution defined by a program, varying only in
whether the variable already exists or not~\citep{rainforth2016bayesian}.

However, in general we are not able to evaluate even the unnormalized density
of a program's outputs due to change-of-variables complications~\citep[Chapter 4]{rainforth2017thesis}.
This creates an ABC-style problem~\citep{csillery2010approximate}, wherein we can
generate weighted samples from the inner query, but we cannot evaluate its density
for a given output.  This creates a substantial computational issue for actually
observing a nested query that must be dealt with on top of any complications from
the nested estimation.  Dealing with these is beyond the scope of this paper and is
left to future work.

\section{DISCRETE OR DETERMINISTIC INPUT VARIABLES}
\label{sec:special}

One special case where consistency can be maintained without requiring infinite computation for each nested call
is when the variables passed to the inner query can only take on, say $C$, finite
possible values. Of particular note, is the case when only deterministic variables are passed to the inner
query, corresponding to $C=1$, which, for example, forms the theoretical basis for the ``programs as proposals'' approach 
of~\cite{cusumano2018using}.
As per Theorem 5 of~\cite{rainforth2017pitfalls}, we can rearrange such problems to
$C$ separate estimators such that the standard Monte Carlo error rate can be achieved.
This is perhaps easiest to see by noting that for such problems, $\int \pi_i(y,z)dz$ can
only on $C$ distinct values, leading to a separate, non nested, inference problem through enumeration.
For repeated nesting, the rearrangement can be recursively applied until one achieves
a complete set of non-nested estimators.  To avoid inferior NMC convergence rates, this special case requires 
explicit rearrangement
or a specialist consideration by the language back-end (as done by e.g. \cite{stuhlmuller2012dynamic,stuhlmuller2014reasoning,cornish2017efficient}).  For example, one can
dynamically catch the inner query being called with the same inputs, e.g. using memoization, and then exploit the fact that all
such cases target the same inference problem.  Care is required in these approaches to ensure the correct
combination with outer query, e.g. returning properly weighted samples and ensuring the budget of the
inner queries remains fixed.

\section{EXACT SAMPLING}
\label{sec:exact}
%
It may, in fact, be possible to provide consistent estimates for many nested query problems without requiring infinite computation
for each nested call by using exact sampling methods such as rejection sampling or coupled Markov chains~\citep{propp1996exact}.
Such an approach is taken by Church~\citep{goodman2008church}, wherein 
no sample ever returns until it passes its 
local acceptance criterion as a hierarchical rejection sampler.  Church is able to do this because it only supports hard conditioning on events with finite probability, allowing it to take a guess-and-check process that produces an exact sample
in finite time, simply sampling from the generative model until the condition is satisfied.
Although the performance still clearly
gets exponentially worse with nesting depth, this is a change in the constant factor of the
computation, not its scaling with the number of samples taken:
generating a single exact sample of the distribution has a finite expected time using rejection sampling which is thus
a constant factor in the convergence rate.

Unfortunately, most problems require conditioning on measure zero events because they include continuous data
-- they require a soft conditioning akin to the inclusion of a likelihood term -- and so cannot be tackled using Church.
Constructing a practical generic exact sampler for soft conditioning in an automated way is likely to be insurmountably problematic
in practice. Nonetheless, it does open up the intriguing prospect
of a hypothetical system that provides a standard Monte Carlo convergence rate for nested inference.
This assertion is a somewhat startling result: it suggests that Monte Carlo 
estimates made using nested exact sampling methods have a fundamentally different
convergence rate for nested inference problems (though not nested estimation problems in general) than, say,
nested self-normalized importance sampling.

\section{CASE STUDY: SIMULATING A POKER PLAYER}
\label{sec:poker}

As a more realistic demonstration of the utility for allowing nested inference
in probabilistic programs, we consider the example of simulating a poker player who reasons about
another player; we
will refer to the two players respectively as P1 and P2.
 Anglican code for this example is given in Figure~\ref{fig:poker-code}.  Though the
model has been kept deliberately simple for exposition, one could easily envisage adapting it to a
higher fidelity simulation.  In particular, one could easily adapt the model to consider multiple players, 
additional betting options for the second player,
and multiple rounds of betting (for which addition levels to the nesting might be required).

\begin{figure*}[t!]
	\begin{lstlisting}[basicstyle=\ttfamily\footnotesize,frame=none]
(defdist hand-strength []
 ;; Samples the strength of a hand
 [dist (uniform-continuous 0 1)]
 (sample* [this] (sample* dist))
 (observe* [this value] (observe* dist value)))

(defdist p1-bet-dist [hand]
 ;; Likelihood model used by player 2 to infer the strength of player
 ;; 1's hand
 [mean-bet (if (< hand 0.5) 0 (* 8 hand))]
 (sample* [this] nil) ;; No need to support sampling here
 (observe* [this value] 
   (log-sum-exp 
    (+ (log 0.95) (observe* (normal mean-bet 2) value))
    (+ (log 0.05) (observe* (uniform-continuous 4 10) value)))))

(with-primitive-procedures [hand-strength p1-bet-dist]
 (defm calc-payoff [p1-hand p1-bet p2-hand p2-call]
  ;; Calculate payoff given actions and hands.
  (let [small-blind 1 
        big-blind 2]
   (case (< p1-bet big-blind)
     true (- small-blind) ;; Lose small blind if fold
     false (case p2-call
             false big-blind ;; Pick up big blinds
             true (if (> p2-hand p1-hand);; Showdown
                    (- p1-bet) 
                    p1-bet)))))

 (defquery p2-sim [p2-hand p1-bet]
  ;; Simulator for player 2 who knows player 1's bet but not her
  ;; hand.  Returns boolean of whether bet is called
  (let [p1-hand (sample (hand-strength))] ;; Simulate a hand for player 1
   (observe (p1-bet-dist p1-hand) p1-bet) ;; Condition on player 1's known bet
   (> p2-hand p1-hand)))

 (defquery p1-payoff [p1-hand p1-bet N_1]
  ;; Estimator for distribution of player 1's payoff for given hand and action
  (let [p2-hand (sample (hand-strength)) ;; Sample hand for opponent
        dist (conditional p2-sim :smc :number-of-particles N_1)
        p2-call (sample (dist p2-hand p1-bet))] ;; Simulate player 2
   (calc-payoff p1-hand p1-bet p2-hand p2-call)))) ;; Return payoff

(defn estimate-payoff [p1-hand p1-bet N_0 N_1]
 ;; Estimates the relative probability of actions given a hand
 (let [samps (->> (doquery :importance p1-payoff [p1-hand p1-bet N_1])
                  (take N_0))]
  (empirical-distribution (collect-results samps))))
	\end{lstlisting}
	\vspace{-10pt}
	\caption{Code simulating the behavior of a poker player who reasons about the behavior of another player.
		Explanation provided in text.\label{fig:poker-code}}
\end{figure*}


\begin{figure*}[t!]
	\centering
	\begin{subfigure}[b]{0.49\textwidth}
		\centering
		\includegraphics[width=\textwidth,trim={6cm 0 4cm 2cm},clip]{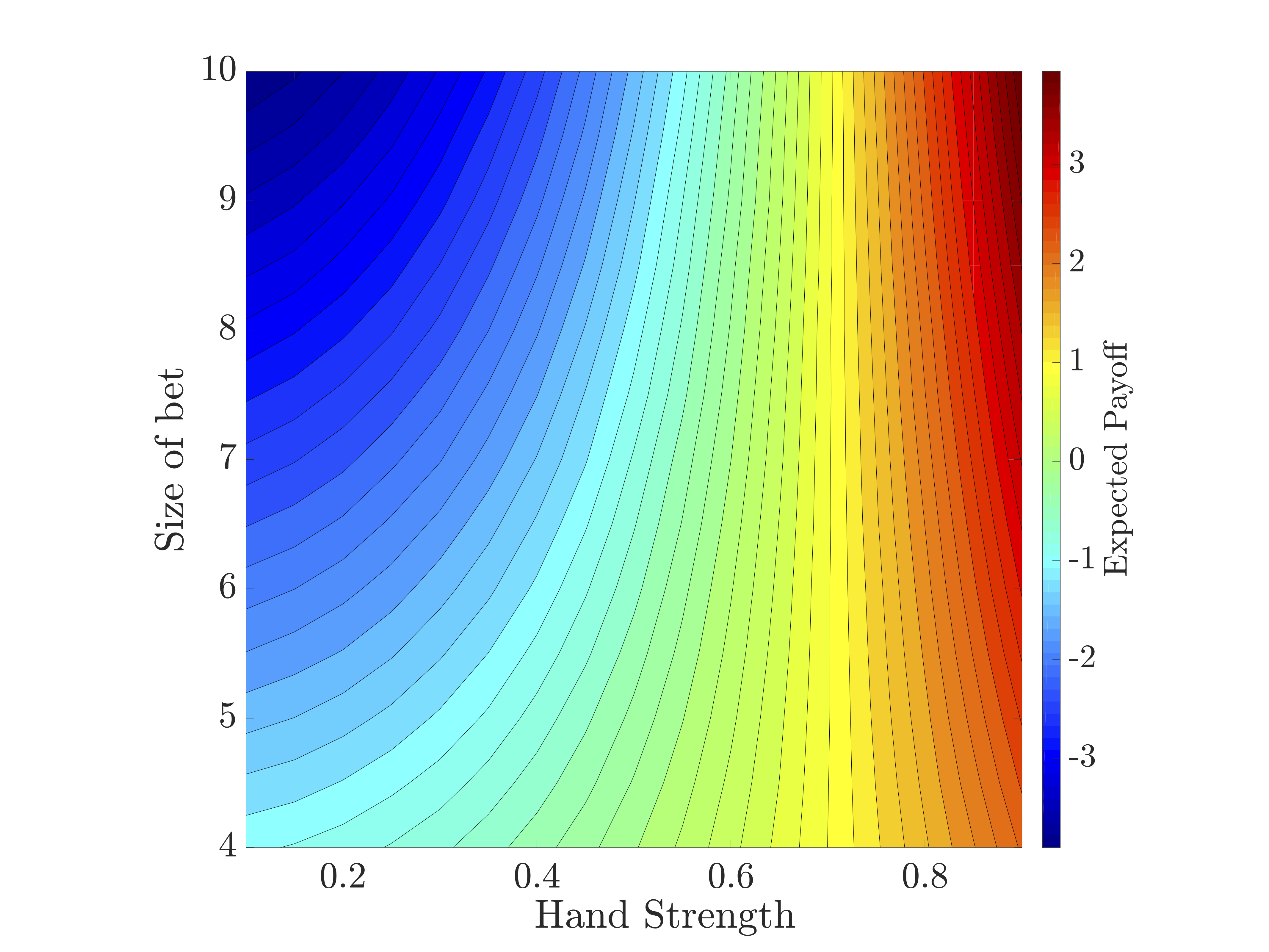}
		\caption{Expected payoff with $N_1=1$ \label{fig:poker1}}
	\end{subfigure}
	\begin{subfigure}[b]{0.49\textwidth}
		\centering
		\includegraphics[width=\textwidth,trim={6cm 0 4cm 2cm},clip]{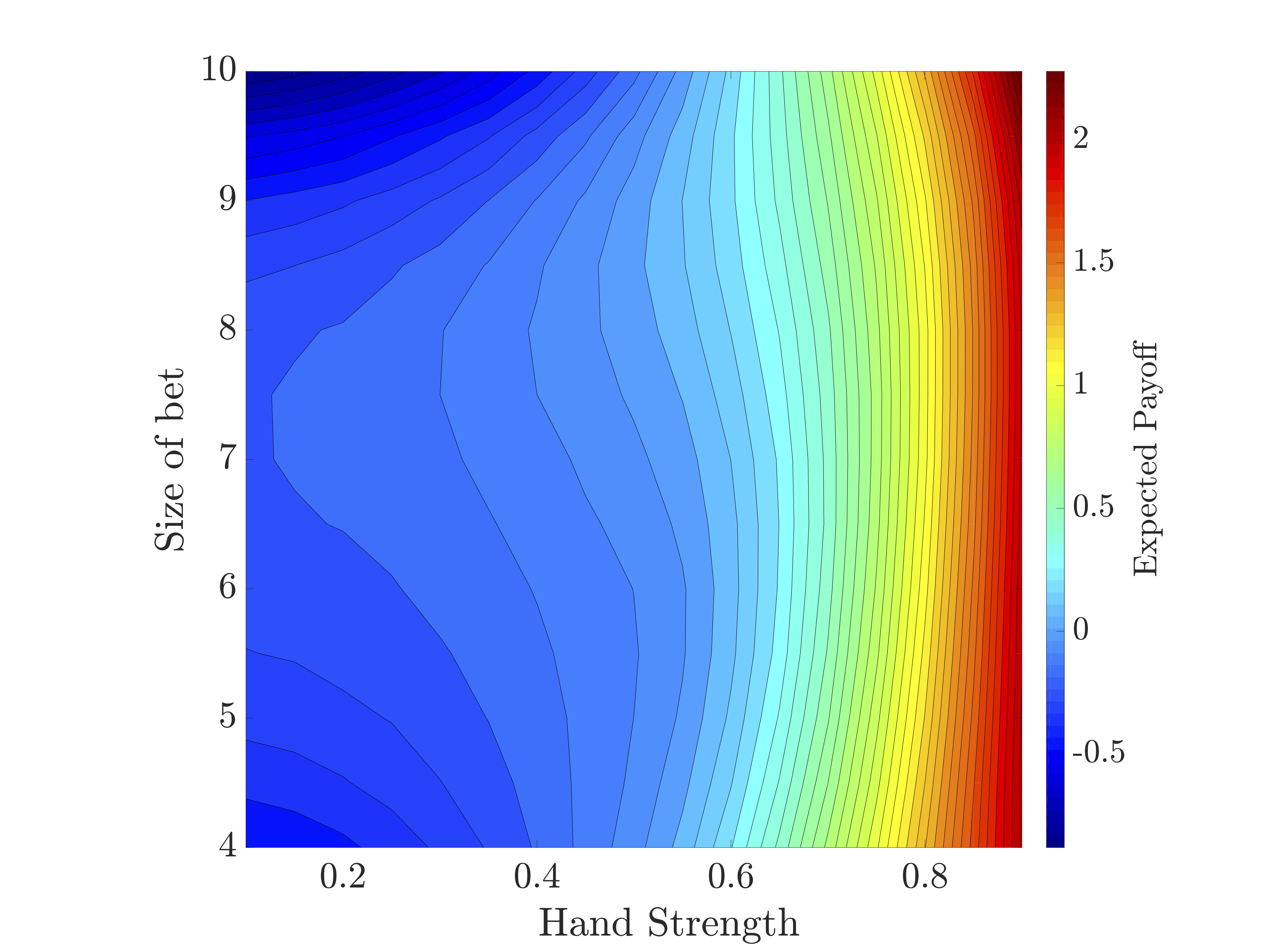}
		\caption{Expected payoff with $N_1=1500$ \label{fig:poker2}}
	\end{subfigure}
	\caption{Contour plots for P1's expected payoff using the poker model given in Figure~\ref{fig:poker-code}
		as a function of  their hand strength and amount bet (in $\pounds$).  On the left is the na\"{i}ve estimator using $N_1=1$,
		which is an equivalent to ignoring the \observe statement in \clj{p2-sim}, such P2 bets when their
		hand is better than one drawn uniformly at random.  On the right is the output of produced by
		using the empirical measured given in~\eqref{eq:emp-measure} based on self-normalized, nested
		importance sampling, with $N_1=1500$.  For both models, an evenly spaced $17\times13$ grid (hand strength
		by bet size) of estimates was calculated using $N_0=2\times10^6$ outer samples, which was
		in turn converted to the shown contour plots using \textsc{matlab}'s \clj{contourf} plot function.
		Note the difference colorbar scaling between the plots.
		\label{fig:poker-line}}
\end{figure*}

At a high level, we a trying to estimate the distribution of payoffs (i.e. the net money received)
by P1 for different hands and bets.  This can then in turn be used to, for example,
optimize the bet made.  The starting situation is that P1 is on the small blind ($\pounds1$) and P2 on the big
blind ($\pounds2$), with no other players currently in the game.  This means that P1 and P2 have already have 
committed (as required by the rules of the game) $\pounds1$ and $\pounds2$ respectively to the pot and it is P1's
turn to act.  P1 can now choose between three actions
\begin{itemize}
\item[Fold] -- P1 declines to commit any more money.  P2 takes the pot giving P1 a payoff of $-\pounds1$.
\item[Call] -- P1 matches the stake from the big blind.  For simplicity, we are
ignoring further rounds of betting and the scenario where P2 makes a further bet. There will, therefore, be a
showdown where the better hand takes the pot.  Here P1's payoff is $+\pounds2$ if they
transpire to have the better hand and $-\pounds2$ otherwise.
\item[Bet] -- P1 increases their stake to between twice the big blind (i.e. $\pounds4$)
and the maximum allow bet size (which we take to be $\pounds10$).  P2 then themselves subsequently
decides whether they will call this bet or fold.  If they fold, P1 receives a payoff of $+\pounds2$.
If they call, a showdown occurs as before, except that the win/lose payoffs are now $\pm$ the size
of P1's bet.
\end{itemize}
Estimating the payoff distributions for the cases where P1 folds or calls is straightforward.
Folding always yields a payoff of $-\pounds1$. Calling yields $+\pounds2$ with probability
equal to the probability that P1's hand is better than a randomly generated hand and $-\pounds2$
otherwise.  Thus if we represent hand strength as a uniform distribution between $0$ and $1$, the expected
payoff of calling when P1 has hand strength $h_1$ becomes simply $2h_1 - 2 (1-h_1)=2(2 h_1-1)$.
Consequently, the expected payoff of calling is better than that of folding if and only if $h_1>0.25$

If P1 instead decides to bet, estimating the payoff distribution becomes substantially more complicated
as it no longer depends only on the respective strength of the two hands, but also the action P2 takes.
This action will be influenced not only by P2's hand, but also the size of P1's bet: P2 can draw inferences
about likely hands for P1 using the information conveyed in P1's bet.  To reflect this, our model for P2,
\clj{p2-sim}, uses a likelihood function for P1's betting, \clj{p1-bet-dist}, to condition on the actual bet made.
This likelihood is based on the, slightly na\"{i}ve, sentiment that P1 will bet more with a better hand,
while also allowing provision for P1 generating their bet at random as a bluff.
P2 decides to call P1's bet if their hand is better than the hand they simulate for P1.  Thus
denoting $c_2$ as the boolean variable indicate if P2 calls then we have
 $P(\{c_2=1\})
= P(h_2 > h_1 | b_1) = \E [h_2>h_1 | b_1]$ where $b_1$ represents P1's bet.\footnote{In practice, 
	it may be more realistic to assume that, rather than aiming to call in proportion to $P(h_2 > h_1 | b_1)$,
	P2 instead tries to directly estimate this probability and deterministically chooses to call
	if this estimate some threshold determined by the pot odds.  This would then lead to an "estimates as variables"
	nested estimation, instead of a nested inference model.}
Note that, from P2's perspective, $h_2$ and $b_1$ are known, but $h_1$ is a random variable.

To estimate the payoff when P1 bets, we nest this model for P2.  Specifically,
the payoff for P1 is given by $\E [\textsc{Payoff}(h_1,b_1,h_2,c_2)]$ where $h_1$ and $b_1$
are fixed, $h_2$ is drawn uniformly at random, and $c_2 | h_2$ is sampled using a nested inference on \clj{p2-sim}.

Keen-eyed readers may have noticed that the use of \conditional in \clj{p1-payoff} is distinct to elsewhere in
the paper as we have explicitly used SMC inference with a provided number of particles $N_1$ for \conditional.  This provides a
roundabout means of controlling the computational budget for calls to \conditional, as we showed is required for convergence in
Section~\ref{sec:samp}.  

Figure~\ref{fig:poker-line} shows contour plots for P1's expected payoff as a function of their hand
strength and amount bet when P2 na\"{i}vely simulates P1's hand strength from the prior (left)
and uses inference to try and infer P1's hand strength from their bet (right).  As expected, for the
na\"{i}ve model then it is better for P1 to make larger bets when she has a strong hand and smaller
bets when she has a weak hand.  When she has a weak hand, the expected payoff of all possible bets is
worse than folding or calling.  

In our nested model, a number of more complex behaviors arise.  Firstly,
we note that the overall variation in expected payoff is less: making significant bets with a weak hand becomes
less detrimental, while the expected rewards of a large bet with a strong hand are also diminished.
This occurs because the act of betting portrays a stronger hand and so P2 is more likely to fold
when they condition their assessment of P1's hand on the fact that P1 bet.  
Consequently, a bluff with a weak hand is more likely to steal the blinds, while a bet with a strong hand
is less likely to get paid off by a call.  In fact, we see that, for this model, it is beneficial to take
a hyper-aggressive stance and always bet: P2 is sufficiently passive that the risk of betting is always worthwhile
even for a very weak hand.

Another, more subtle, effect that transpires is that, when P1 has a weak hand, it is possible to both
bet too much and bet too little.  Too small a bet is more likely to
get called -- even when P2 has a weak hand, they are being offered very favorable odds to call the
bet in hope that P1 is bluffing.  Too large a bet exposes P1 to unnecessarily large losses when P2
transpires to have a strong hand and decides to call.  A medium sized bluff thus offers the best balance
between being believable and not being unnecessarily risky.
A different effect is seen when P1 has a strong hand: small bets are likely to get paid-off by
a large number of hands, while large bets may yield large rewards or potentially cause an even stronger
hand to fold.  Thus a mid-level bet actually becomes the worst option.

For this problem, nesting has allowed us to emulate a player that assumes simplistic play from
their opponent to outsmart them.  One could clearly envisage making the model even smarter by adding
additional layers of nesting.  Suppose that P2 is actually a good player that more explicitly
reasons about the fact that P1 will be reasoning about them.  We could then, for example, use
the model developed so far for P2's simulation of P1, meaning that they will be more attuned to the fact that
P1 might be bluffing.  Amongst other things, this
is then likely to make them more likely to call, knowing that P1 is playing an aggressive game and they have
 a good chance of catching a bluff.  P1
could then in turn use this a higher fidelity model for P2, replacing the current \clj{p-sim}.  It is
easy to see how such a meta-reasoning hierarchy could potentially lead to smarter and smarter players.  However, the NMC 
converges rates tell us that doing so comes at a substantial cost in terms of the difficultly
of solving the resulting nested estimation problem: the required number of samples increases
exponentially with the depth of the nesting.

\begin{figure}[t]
	\centering
	{\includegraphics[width=0.95\textwidth,trim={0.5cm 0 4cm 1cm},clip]{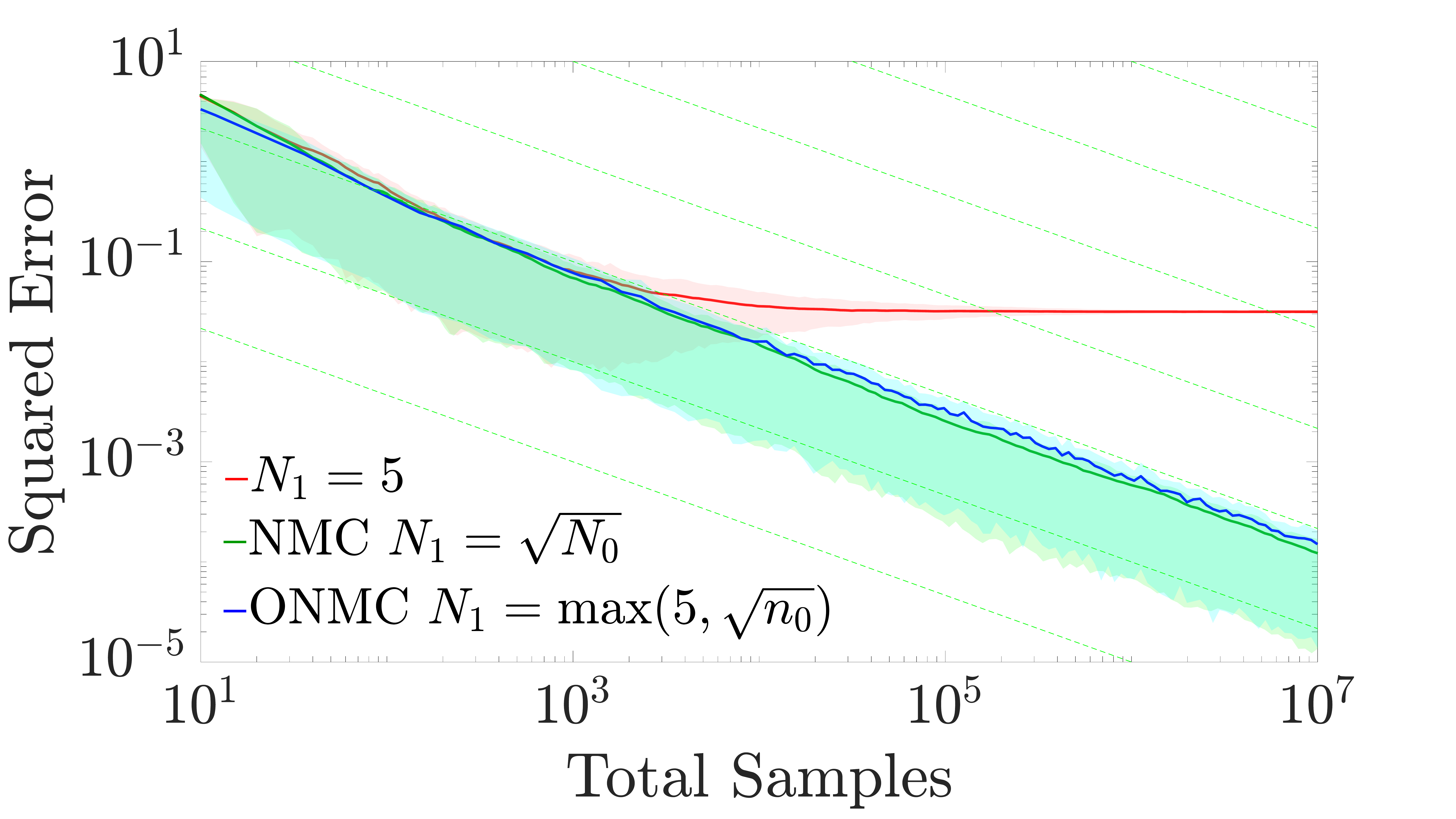}}
	\vspace{-8pt}
	\caption{Convergence of ONMC, NMC, and fixed $N_1$ for expected payoff in poker example
		with hand strength set to $0.1$ and bet size set to $6$.  Results are averaged over $1000$ runs, with 
		solid lines showing the mean and shading the 25-75\% quantiles.  
		Ground truth was estimated empirically using a large scale NMC run
		with $N_0 = 5 \times 10^7$ and $N_1 = 5000$.
		The theoretical	rates for NMC are shown by the dashed lines.
		\label{fig:onmc-poker}		\vspace{-16pt}}
\end{figure}

%

\begin{figure*}[t!]
	\centering		
	\rule{\linewidth}{0.4pt}
	\vspace{-22pt}		
	\begin{lstlisting}[basicstyle=\ttfamily\footnotesize,multicols=2,frame=none]
(defm prior [] (normal 0 1))
(defm lik [theta d] (normal theta d))
	
(defquery inner-q [y d]
 (let [theta (sample (prior))]
  (observe (lik theta d) y)))
	
(defn inner-E [y d M]
 (->> (doquery :importance 
        inner-q [y d])
      (take M)
      log-marginal))
	
(with-primitive-procedures [inner-E]
 (defquery outer-q [d M]
  (let [theta (sample (prior))
        y (sample (lik theta d))
        log-lik (observe* 
                 (lik theta d) y)
        log-marg (inner-E y d M)]
   (- log-lik log-marg))))
	
(defn outer-E [d M N]
 (->> (doquery :importance 
        outer-q [d M])
      (take N)
      collect-results
      empirical-mean))
	\end{lstlisting}
	\vspace{-14pt}		
	\rule{\linewidth}{0.4pt}
	\caption{Anglican code for Bayesian experimental design.
		By changing the definitions of \clj{prior}
		and \clj{lik}, this code can be used as a
		NMC estimator (consistent as $N,M\rightarrow\infty$) for any static
		Bayesian experimental design problem.
		Here \lsi{observe*} is a function for returning the log likelihood (it does not
		affect the trace probability), \lsi{log-marginal} produces a partition function estimate
		from a collection of weighted samples,
		and \lsi{->>} successively applies a series of functions calls,
		using the result of one as the last input the next.  When \lsi{outer-E} is invoked,
		this runs importance sampling on \lsi{outer-q}, which, in addition to carrying out
		its own computation, calls \lsi{inner-E}.  This, in turn, invokes another inference over
		\lsi{inner-q}, such that a \mc estimate using \lsi{M} samples is constructed for
		each sample of \lsi{outer-q}.  Thus \lsi{log-marg} is \mc estimate itself.
		The final return is the (weighted) empirical mean for
		the outputs of \lsi{outer-q}. 
		\label{fig:nest:exp}}	
\end{figure*}

We finish by comparing the empirical performance of ONMC and NMC for this particular problem.
Here we consider a fixed bet of $\pounds 6$ and hand strength of $0.1$.  The convergence, shown
in Figure~\ref{fig:onmc-poker}, demonstrates extremely similar performance for the two approaches,
while again highlighting the danger of keeping $N_1$ fixed.  Note that the slightly different
setup used for $\tau$ to that used in the Gaussian example does not make any noticeable difference
to the performance (not shown), with the different choice stemming from a desire to better highlight the problem
of keeping $N_1$ fixed.

\section{SIMPLE ANALYTICAL MODEL DETAILS}
\label{sec:simple-model}

We consider the following simple analytic model introduced by~\citep{rainforth2017pitfalls}
for which the true nested expectation is
$\gamma_0 = \frac{1}{2}\log\left(\frac{2}{5\pi}\right)-\frac{2}{15}$
\begin{subequations}
	\label{eq:simple-model}
	\begin{align}
	&y^{(0)} \sim \mathrm{Uniform}(-1,1), \\
	&y^{(1)}  \sim \mathcal{N}(0,1), \\
	&f_1(y^{(0)},y^{(1)}) = \sqrt{\frac{2}{\pi}}\exp\left(-2(y^{(0)}-y^{(1)})^2\right), \\
	&f_0(y^{(0)},\gamma_1(y^{(0)})) = \log (\gamma_1 (y^{(0)})).
	\end{align}
\end{subequations}
Results for this model are shown in Figure~\ref{fig:onmc} in the main paper.

\section{EXPERIMENTAL DESIGN EXAMPLE}
\label{sec:exp}

An example application of using estimates as first class variables if provided by
Bayesian experimental design~\citep{chaloner1995bayesian}.
One can implicitly use expectation estimates as first class variables in Anglican 
by either calling \doquery inside a \defdist declaration or in a \defn function
passed to a query using \clj{with-primitive-procedures}, a macro providing
the appropriate wrappings to convert a Clojure function to an Anglican one.
Anglican code using the latter approach to create generic estimator for 
Bayesian experimental design problems is shown in
Figure~\ref{fig:nest:exp}, providing a consistent means of carrying out this class of nested estimation problems.  
\cite[Figure 6]{rainforth2017pitfalls} shows the convergence code equivalent to that of 
Figure~\ref{fig:nest:exp} for a delay discounting model.
This shows the convergence (or more specifically lack there of) in the case where $M=N_1$ is held fixed
and the superior convergence achieved when exploiting the finite number of possible outputs to
produce a reformulated, standard Monte Carlo, estimator. 
It therefore highlights both the importance
of increasing the number of samples used by the inner query and exploiting our outlined special cases
when possible.

\section*{Acknowledgements}
I would like to
thank Yee Whye Teh, N. Siddharth, and Benjamin Bloem-Reddy
for feedback on drafts of this work.
My research leading to these results has received funding from the
European Research Council under the European Union's Seventh Framework
Programme (FP7/2007-2013) ERC grant agreement no. 617071.  Some of the work was 
undertaken while I was at the Department of Engineering Science and was supported 
by a BP industrial grant.

\newpage

\bibliography{refs}

\end{document}